\documentclass[11pt]{article}

\usepackage[utf8]{inputenc} 
\usepackage{geometry} 
\usepackage{fancyhdr} 
\usepackage{amsthm}
\usepackage{authblk}
\usepackage{parskip}
\usepackage[american]{babel}
\usepackage[title]{appendix}

\usepackage{natbib} 
\usepackage{mathtools} 
\usepackage{booktabs} 
\usepackage{tikz} 

\usepackage{amsmath,amsfonts,graphicx, amssymb, amsthm}
\usepackage{mathtools, bm,bbm}
\usepackage{color}
\usepackage{multicol}
\usepackage{hyperref}       
\usepackage{url}            
\usepackage{booktabs}       
\usepackage{nicefrac}       
\usepackage{algorithm}
\usepackage[noend]{algorithmic}
\usepackage{subfigure}
\hypersetup{colorlinks=true}
\usepackage{thm-restate}

\def\ddefloop#1{\ifx\ddefloop#1\else\ddef{#1}\expandafter\ddefloop\fi}
\def\ddef#1{\expandafter\def\csname bb#1\endcsname{\ensuremath{\mathbb{#1}}}}
\ddefloop ABCDEFGHIJKLMNOPQRSTUVWXYZ\ddefloop
\def\ddef#1{\expandafter\def\csname c#1\endcsname{\ensuremath{\mathcal{#1}}}}
\ddefloop ABCDEFGHIJKLMNOPQRSTUVWXYZ\ddefloop

\DeclareMathOperator*{\argmin}{arg\,min}

\def\R{\mathbb{R}}

\def\N{\mathbb{N}}
\def\C{\mathbb{C}}
\def\E{\mathbb{E}}

\def\bSig{\bm{\Sigma}}

\def\var{\text{Var}}

\def\wh{\widehat}

\def\cN{\mathcal{N}}

\DeclarePairedDelimiterX{\ip}[2]{\langle}{\rangle}{#1, #2}

\def\lm{\lambda_{\text{max}}}

\newcommand{\p}[1]{\left( #1 \right)}
\newcommand{\abs}[1]{\left| #1 \right|}

\renewcommand {\AA}  { {\mathbf{A}} }
\newcommand {\RR}  { {\mathbf{R}} }
\newcommand {\BB}  { {\mathbf{B}} }
\newcommand {\CC}  { {\mathbf{C}} }
\newcommand {\DD}  { {\mathbf{D}} }
\newcommand {\EE}  { {\mathbf{E}} }

\newcommand {\HH}  { {\mathbf{H}} }

\newcommand {\UU}  { {\mathbf{U}} }
\newcommand {\VV}  { {\mathbf{V}} }
\newcommand {\bt}  { {\mathbf{t}} }
\newcommand {\bs}  { {\mathbf{s}} }
\newcommand {\TT}  { {\mathbf{T}} }
\newcommand {\GG}  { {\mathbf{G}} }

\newcommand {\MM}  { {\mathbf{M}} }

\newcommand {\NN}  { {\mathbf{N}} }

\newcommand {\QQ}  { {\mathbf{Q}} }
\newcommand {\bSS}  { {\mathbf{S}} }

\newcommand {\XX}  { {\mathbf{X}} }
\newcommand {\YY}  { {\mathbf{Y}} }
\newcommand{\eye}{\mathbf{I}}

\newcommand {\zz}  { {\bf z} }
\newcommand {\bgg}  { {\bf g} }
\newcommand {\xx}  { {\bf x} }
\newcommand {\yy}  { {\bf y} }

\renewcommand {\aa}  { {\bf a} }

\newcommand {\qq}  { {\bf q} }
\newcommand {\pp}  { {\bf p} }
\newcommand {\uu}  { {\bf u} }
\newcommand {\vv}  { {\bf v} }
\newcommand {\ww}  { {\bf w} }

\newcommand {\bb}  { {\bf b} }

\newcommand {\ee}  { {\bf e} }

\newcommand {\zero}  { {\bf 0} }

\newcommand {\nnz}  { {\bf nnz} }
\newcommand {\poly}  { {\bf poly} }

\providecommand{\qnorm}[1]{\left \vert \mspace{-1.8mu} \left\vert
\mspace{-1.8mu} \left \lvert #1 \right \vert \mspace{-1.8mu} \right\vert
\mspace{-1.8mu} \right\vert}
\usepackage{optidef}

\usepackage[capitalise]{cleveref}
\crefname{equation}{}{}
\crefname{figure}{Figure}{Figures}
\creflabelformat{equation}{\textup{(#2#1#3)}}
\crefname{assumption}{Assumption}{Assumptions}
\crefname{condition}{Condition}{Conditions}

\newtheorem{theorem}{Theorem}
\newtheorem{lemma}{Lemma}

\newtheorem{definition}{Definition}
\newtheorem{remark}{Remark}
\newtheorem{fact}{Fact}
\usepackage{xspace}

\renewcommand\th{\textsuperscript{th}\xspace}

\usepackage[inline]{enumitem}
\setlist[enumerate,1]{leftmargin=*, label = {\bfseries \arabic*.}}
\setlist[itemize,1]{leftmargin=*}
\newcommand{\real}{\mathbb{R}}

\newcommand*\dotprod[1]{\left\langle #1\right\rangle}
\newcommand*\vnorm[1]{\left\| #1\right\|}
\renewcommand{\Pr}{\hbox{\bf{Pr}}}

\usepackage{dsfont}

\newcommand {\alphak}  { {{\alpha}_{k}} }

\newcommand {\HHk}  { {\HH_{k}} }

\newcommand {\bggk}  { {{\bgg}_{k}} }
\newcommand {\bggkk}  { {{\bgg}_{k+1}} }
\newcommand {\xxo}  { {{\xx}_{0}} }

\newcommand {\xxk}  { {{\xx}_{k}} }

\newcommand {\xxkk}  { {{\xx}_{k+1}} }
\newcommand {\xxs}  { {{\xx}^{\star}} }

\newcommand {\ppk}  { {{\pp}_{k}} }

\newcommand {\diag}  { {\textnormal{diag}} }
\makeatletter
\newcommand*{\transpose}{%
	{\mathpalette\@transpose{}}%
}
\newcommand*{\@transpose}[2]{%
	\raisebox{\depth}{$\m@th#1\intercal$}%
}
\makeatother
\newcommand*{\hermconj}{{\mathsf{\ast}}}
\newcommand {\HHdk}  { {\left[\HH_{k}\right]^{\dagger}} }
\newcounter{comment}

\title{\bf{{Non-PSD Matrix Sketching with Applications to Regression and Optimization}}}

\author[1]{Zhili Feng}
\author[2]{Fred Roosta}
\author[3]{David P. Woodruff}
\affil[1]{%
\small{
    Machine Learning Department\\
    Carnegie Mellon University\\
}
}
\affil[2]{%
  School of Mathematics and Physics\\
  University of Queensland\\
}
\affil[3]{
    Computer Science Department\\
    Carnegie Mellon University\\
}

\begin{document}
\maketitle

\begin{abstract}
A variety of dimensionality reduction techniques have been applied for computations involving large matrices. The underlying matrix is randomly compressed into a smaller one, while approximately retaining many of its original properties. As a result, much of the expensive computation can be performed on the small matrix. The sketching of positive semidefinite (PSD) matrices is well understood, but there are many applications where the related matrices are not PSD, including Hessian matrices in non-convex optimization and covariance matrices in regression applications involving complex numbers. In this paper, we present novel dimensionality reduction methods for non-PSD matrices, as well as their ``square-roots", which involve matrices with complex entries.  We show how these techniques can be used for multiple downstream tasks. In particular, we show how to use the proposed matrix sketching techniques for both convex and non-convex optimization,  $\ell_p$-regression for every $1 \leq p \leq \infty$, and vector-matrix-vector queries.
\end{abstract}

\section{Introduction}
Many modern machine learning tasks involve massive datasets, where an input matrix $\AA\in\R^{n\times d}$ is such that $n\gg d$. In a number of cases, $\AA$ is highly redundant. For example, if we want to solve the ordinary least squares problem $\min_{\xx}\|\AA\xx-\bb\|_2^2$, one can solve it exactly given only $\AA^T \AA$ and $\AA^T \bb$. To exploit this redundancy, numerous techniques have been developed to reduce the size of $\AA$. Such dimensionality reduction techniques are used to speed up various optimization tasks and are often referred to as \textsl{sketching}; for a survey, see \citet{woodruff2014sketching}. 

A lot of previous work has focused on sketching PSD matrices. For example, the Hessian matrices in convex optimization \citep{xu2016sub}, the covariance matrices $\XX^\top \XX$ in regression over the reals, and quadratic form queries $\xx^\top\AA\xx$ \citep{andoni2016sketching}. Meanwhile, less is understood for non-PSD matrices. These matrices are naturally associated with complex matrices: the Hessian of a non-convex optimization problem can be decomposed into $\HH=\XX^\top \XX$ where $\XX$ is a matrix with complex entries, and a complex design matrix $\XX$ has a non-PSD covariance matrix.
However, almost all sketching techniques were developed for matrices with entries in the real field $\R$. While some results carry over to the complex numbers $\C$ (e.g., \citet{tropp2015introduction} develops concentration bounds that work for complex matrices), many do not and seem to require non-trivial extensions. In this work, we show how to efficiently sketch non-PSD matrices and extend several existing sketching results to the complex field. We also show how to use these in optimization, for both convex and non-convex problems, the sketch-and-solve paradigm for complex $\ell_p$-regression with $1 \leq p \leq \infty$, as well as vector-matrix-vector product queries. 

\paragraph{Finite-sum Optimization.} 
We consider optimization problems of the form
\begin{align}
\label{eq:obj_sum}
\min_{\xx \in \real^d} F(\xx) \triangleq \frac{1}{n} \sum_{i=1}^n f_i(\aa_{i}^{\transpose}\xx) + r(\xx),
\end{align}
where $ n \gg d \geq 1 $, each $ f_{i}:\real \rightarrow \real $ is a smooth but possibly non-convex function, $r(\xx)$ is a regularization term, and $\aa_i \in \real^d, i = 1,\ldots, n,$ are given. Problems of the form~\eqref{eq:obj_sum} are abundant in machine learning \citep{shalev2014understanding}. 
Concrete examples include robust linear regression using Tukey's biweight loss \citep{beaton1974fitting}, i.e., $f_{i}(\dotprod{\aa_{i},\xx}) = {\left(\aa_i^{\transpose}\xx - b_{i}\right)^{2}}/{(1+\left(\aa_i^{\transpose}\xx - b_{i}\right)^{2})}$, where $ b_{i} \in \real $, and non-linear binary classification \citep{xuNonconvexEmpirical2017}, i.e., $f_{i}(\dotprod{\aa_{i},\xx}) = \left({1}/{\left(1+\exp\left(-\aa_i^{\transpose}\xx\right) \right)} - b_{i}\right)^{2}$, where $ b_{i} \in \left\{0,1\right\} $ is the class label.
By incorporating curvature information, second-order methods are gaining popularity over first-order methods in certain applications.
However, when $ n \gg d \geq 1 $, operations involving the Hessian of $ F $ constitute a computational bottleneck. To this end, randomized Hessian approximations have shown great success in reducing  computational complexity (\citep{roosta2019sub,xu2016sub,xuNonconvexTheoretical2017,pilanci2017newton,erdogdu2015convergence,bollapragada2019exact}).

In the context of \cref{eq:obj_sum}, it is easy to see that the Hessian of $ F $ can be written as
$\nabla^2 F(\xx) = \sum_{i=1}^n f_{i}^{\prime \prime}(\aa_i^{\transpose}\xx)\aa_i\aa_i^{\transpose}/n + \nabla^{2} r(\xx) = \AA^{\transpose} \DD(\xx) \AA/n  + \nabla^{2} r(\xx) $, where
$
\AA^{\transpose} = \begin{bmatrix}
\aa_1,\ldots,\aa_n
\end{bmatrix} \in \real^ {d \times n}$ and $
\DD(\xx) = \diag \left[f_1''(\aa_1^{\transpose}\xx) ~ f_2''(\aa_2^{\transpose}\xx) ~ \ldots ~ f_n''(\aa_n^{\transpose}\xx)\right] \in \real^{n \times n}.
$
Of particular interest in this work is the application of randomized matrix approximation techniques  \citep{woodruff2014sketching,mahoney2011randomized,drineas2016randnla}, in particular, constructing a random sketching matrix $\bSS$ to ensure that $\HH(\xx) \triangleq \AA^\transpose \DD^{1/2}\bSS^\transpose\bSS\DD^{1/2}\AA + \nabla^{2} r(\xx) \approx \AA^\transpose \DD\AA/n + \nabla^{2} r(\xx) =\nabla^2F(\xx)$. 
Notice that $\DD^{1/2}\AA$ may have complex entries if $f_i$ is non-convex.

\paragraph{The Sketch-and-Solve Paradigm for Regression.}
In the overconstrained least squares regression problem, the task is to solve $\min_\xx\|\AA\xx-\bb\|$ for some norm $\|\cdot\|$, and here we focus on the wide class of $\ell_p$-norms, where for a vector $y$, $\|y\|_p =  (\sum_j |y_j|^p  )^{1/p}$. Setting the value $p$ allows for adjusting the sensitivity to outliers; for $p < 2$ the regression problem is often considered more robust than least squares because one does not square the differences, while for $p > 2$ the problem is considered more sensitive to outliers than least squares. The different $p$-norms also have statistical motivations: for instance, the $\ell_1$-regression solution is the maximum likelihood estimator given i.i.d. Laplacian noise. Approximation algorithms based on sampling and sketching have been thoroughly studied for $\ell_p$-regression, see, e.g.,  \citep{clarkson2005subgradient,clarkson2016fast,dasgupta2009sampling,meng2013low,sohler2011subspace,woodruff2013subspace,clarkson2017low,wang2019tight}. These algorithms typically follow the sketch-and-solve paradigm, whereby the dimensions of $\AA$ and $\bb$ are reduced, resulting in a much smaller instance of $\ell_p$-regression, which is tractable. In the case of $p = \infty$, sketching is used inside of an optimization method to speed up linear programming-based algorithms \citep{cohen2019solving}. 

To highlight some of the difficulties in extending $\ell_p$-regression algorithms to the complex numbers, consider two popular cases, of $\ell_1$ and $\ell_{\infty}$-regression. The standard way of solving these regression problems is by formulating them as linear programs. However, the complex numbers are not totally ordered, and linear programming algorithms therefore do not work with complex inputs. Stepping back, what even is the meaning of the $\ell_p$-norm of a complex vector $y$? In the definition above $\|y\|_p =  (\sum_j |y_j|^p  )^{1/p}$, and $|y_j|$ denotes the modulus of the complex number, i.e., if $y_j = a + b \cdot i$, where $i = \sqrt{-1}$, then $|y_j| = \sqrt{a^2 + b^2}$. Thus the $\ell_p$-regression problem is really a question about minimizing the $p$-norm of a sum of Euclidean lengths of vectors. As we show later, this problem is very different than $\ell_p$ regressions over the reals.

\paragraph{Vector-matrix-vector queries.} Many applications require queries of the form $\uu^\top\MM\vv$, which we call  vector-matrix-vector queries, see, e.g., \cite{rashtchian2020vector}. For example, if $\MM$ is the adjacency matrix of a graph, then $\uu^\top\MM\vv$ answers whether there exists an edge between pair $\{\uu,\vv\}$. These queries are also useful for independent set queries, cut queries, etc. Many past works have studied how to sketch positive definite $\MM$ (see, e.g., \cite{andoni2016sketching}), but it remains unclear how to handle the case when $\MM$ is non-PSD or has complex entries.   

\paragraph{Contributions.} 
We consider non-PSD matrices and their "square-roots", which are complex matrices, in the context of optimization and the sketch-and-solve paradigm. Our goal is to provide tools for handling such matrices in a number of different problems, and
to the best of our knowledge, is the first work to systematically study dimensionality reduction techniques for such matrices. 

For optimization of \cref{eq:obj_sum}, where each $f_{i}$ is potentially non-convex, we investigate non-uniform data-aware methods to construct a sampling matrix $ \bSS $ based on a new concept of leverage scores for complex matrices. In particular, we propose a hybrid deterministic-randomized sampling scheme, which is shown to have important properties for optimization. We show that our sampling schemes can guarantee appropriate matrix approximations (see \cref{eq:H,eq:S}) with competitive sampling complexities. Subsequently, we investigate the application of such sampling schemes in the context of convex and non-convex Newton-type methods for \cref{eq:obj_sum}.

For complex $\ell_{p}$-regression, we use Dvoretsky-type embeddings as well as an isometric embedding from $\ell_1$ to $\ell_{\infty}$ to construct oblivious embeddings from an instance of a complex $\ell_p$-regression problem to a real-valued $\ell_p$-regression problem, for $p \in [1, \infty]$. Our algorithm runs in $\cO((\nnz(\AA)+\poly(d/\epsilon)))$ time for constant $p\in[1,\infty)$, and $\cO(\nnz(\AA)2^{\tilde{O}(1/\epsilon^2)})$ time for $p=\infty$. Here $\nnz(\AA)$ denotes the number of
non-zero entries of the matrix $\AA$. 

For vector-matrix-vector queries, we show that if the non-PSD matrix has the form $\MM=\AA^\top\BB$, then we can approximately compute $\uu^\top\MM\vv$ in just $\cO(\nnz(\AA)+n/\epsilon^2)$ time, whereas the na\"ive approach takes $nd^2+d^2+d$ time. 

\paragraph{Notation.}
Vectors and matrices are denoted by bold lower-case and bold upper-case letters, respectively, e.g., $ \vv $ and $ \VV $. 
We use regular lower-case and upper-case letters to denote scalar constants, e.g., $ d $  or $ L $. 
For a complex vector $ \vv $, its real and conjugate transposes are respectively denoted by $ \vv^{\transpose} $ and $ \vv^{\hermconj} $. 
For two vectors $ \vv,\ww $, their inner-product is denoted by $ \dotprod{\vv, \ww}$. For a vector $\vv$ and a matrix $ \VV $, $ \|\vv\|_p $, $ \|\VV\| $, and $\|\VV\|_F$ denote vector $ \ell_{p} $ norm, matrix spectral norm, and Frobenius norm, respectively. For $\|\vv\|_2$, we write $\|\vv\|$ as an abbreviation. Let $|\VV|$ denote the entry-wise modulus of matrix $\VV$.
Let $\VV_{i,j}$ denote the $(i, j)$-th entry, $\VV_{i}=\VV_{i, *}$ be the $i$-th row, and $\VV_{*, j}$ be the $j$-th column.
The iteration counter for the main algorithm appears as a subscript, e.g., $ \ppk $. 
For two symmetric matrices $ \AA $ and $ \BB $, the L\"{o}wner partial order $\AA \succeq \BB$ indicates that $ \AA-\BB $ is symmetric positive semi-definite.
$ \AA^{\dagger} $ denotes the Moore-Penrose generalized inverse of matrix $ \AA $.  
For a scalar $d$, we let $\poly(d)$ be a polynomial in $d$. We let $\text{diag}(\cdot)$ denote a diagonal matrix.

Here we give the necessary definitions.

	\begin{definition}[Well-conditioned basis and $\ell_p$ leverage scores]
		An $n\times d$ matrix $\UU$ is an $(\alpha,\beta,p)$-well-conditioned basis for the column span of $\AA$ if
		\begin{enumerate*}[series = tobecont, itemjoin = \quad, label=(\roman*)]
			\item 
			$
	    	(\sum_{i\in[n]}\sum_{j\in[d]}|\UU_{ij}|^p)^{1/p}\leq \alpha.
			$
			\item 
			$
				\text{For all } \xx\in\R^d, \|\xx\|_q\leq\beta\|\UU\xx\|_p, \text{ where } 1/p+1/q=1.
			$
			\item
			The column span of $\UU$ is equal to the column span of $\AA$.
		\end{enumerate*}
		For such a well conditioned basis, $\|\UU_{i*}\|_p^p$ is defined to be the $\ell_p$ leverage score of the $i$-th row of $\AA$. The $\ell_p$ leverage scores are \emph{not} invariant to the choice of well-conditioned basis. 
	\end{definition}
	
	\begin{definition}[$\ell_p$ Auerbach Basis]
	An Auerbach basis $\AA$ of $\UU\in \R^{n\times d}$ is such that:
	\begin{enumerate*}[series = tobecont, itemjoin = \quad, label=(\roman*)]
		\item 
			$
				\text{span}(\UU)=\text{span}(\AA).
			$
		\item 
			For all $j\in[d]$, $\|\AA_{*j}\|_p=1$.
		\item 
			For all $\xx\in\R^d$, $d^{-1/q}\|\xx\|_q\leq\|\xx\|_\infty\leq \|\AA\xx\|_p$, where $1/p+1/q=1$
	\end{enumerate*} 
\end{definition}

	\begin{definition}[$\ell_p$-subspace embedding]
    Let $\AA\in\R^{n\times d}$, $\bSS\in\C^{s\times n}$. We call $\bSS$ an $\epsilon$ $\ell_p$-subspace embedding if for all $\xx\in\C^d$, $\|\AA\xx\|_p \le \|\bSS \AA\xx\|_p \le  (1+\epsilon)\|\AA\xx\|_p$.
\end{definition}

\section{Sketching Non-PSD Hessians for Non-Convex Optimization}
\label{SEC:OPTIMIZATION}

We first present our sketching strategies and then apply them to an efficient solution to \cref{eq:obj_sum} using different optimization algorithms. All the proofs are in the supplementary material.

\subsection{Complex Leverage Score Sampling}
\label{SEC:OPTIMIZATION_LS}
\begin{algorithm}[H]
	\caption{Construct Leverage Score Sampling Matrix \label{alg:lssampling}}
	\begin{algorithmic}[1]
		\STATE \textbf{Input:} $\DD^{1/2}\AA\in\C^{n\times d}$, number $s$ of samples, empty matrices $\RR\in\R^{s\times s}$ and $\bm{\Omega}\in\R^{n\times s}$
		\STATE Compute SVD of $\DD^{1/2}\AA = \UU\bSig\VV^*$
		\FOR {$i\in[n]$}
		\STATE Calculate the $i^{th}$ leverage score         $\ell_i=\|\UU_{i,*}\|^2$
		\ENDFOR
		\FOR {$ j\in[s] $} 
		\STATE Pick row $i$ independently and with replacement with probability $p_i=\frac{\ell_i}{\sum_i\ell_i}$
		\STATE Set $\bm{\Omega}_{i,j}=1$ and $\RR_{i,i}=\frac{1}{\sqrt{sp_i}}$
		\ENDFOR
		\STATE \textbf{Output:} $ \bSS=\RR\cdot\bm{\Omega}^\transpose $
	\end{algorithmic}
\end{algorithm}

It is well-known that leverage score sampling gives an $\epsilon$ $\ell_2$-subspace embedding for real matrices with high probability, see, e.g., \cite{woodruff2014sketching}. Here we extend the result to the complex field:
\begin{theorem}\label{THM:LOWNERRESULT}
	For $i\in[n]$, let $\tilde\ell_i\geq\ell_i$ be a constant overestimate to the leverage score of the $i^{th}$ row of $\BB\in\C^{n\times d}$. Assume that $\BB^\transpose \BB\in\R^{d\times d}$. Let $p_i={\tilde\ell_i}/{\sum_{j\in[n]}\tilde\ell_j}$ and $t= cd \epsilon^{-2} \log(d/\delta)$ for a large enough constant $c$. We sample $t$ rows of $A$ where row $i$ is sampled with probability $p_i$ and rescaled to ${1}/{\sqrt{tp_i}}$. Denote the sampled matrix by $\CC$. Then with probability $1-\delta$, $\CC$ satisfies:
	$
	\BB^\transpose  \BB -\epsilon \BB^*\BB\preceq\CC^\transpose \CC \preceq \BB^\transpose \BB +\epsilon \BB^*\BB.
	$
\end{theorem}

\begin{theorem}\label{THM:NORMRESULT}
	Under the same assumptions and notation in \cref{THM:LOWNERRESULT}, let $t= c d \gamma \epsilon^{-2} \log(d/\delta)$, where
	$
	\gamma=\vnorm{\sum_{i\in[n]}\frac{\|\BB_i\|^2}{\tilde\ell_i}\BB_i^*\BB_i}.
	$
	Then with probability $1-\delta$, $\CC$ satisfies $\|\CC^\transpose\CC-\BB^\transpose\BB\|<\epsilon$.
\end{theorem}

\begin{remark}\label{remark:sampleupperboundcomparison1}
	 It is hard to directly compare \cref{THM:NORMRESULT} to the sample complexity of \citet{xuNonconvexTheoretical2017}, where they require $t\geq \frac{\wh K^2}{\epsilon^2}\log(2d/\delta)$, $\widehat K\geq \|\BB^*\BB\|=\cO(\sigma_1^2)$. To apply \cref{THM:NORMRESULT} to a Hessian of the form $\AA^\transpose \DD\AA$, one should set $\BB = \DD^{1/2}\AA$. Compared to the previously proposed sketching $\AA^\transpose \bSS^\transpose \bSS\DD\AA$ for non-convex $F$, our proposed sketching $\AA^\transpose \DD^{1/2}\bSS^\transpose \bSS\DD^{1/2}\AA$ in practice often has better performance (see \cref{SEC:EXP}). We conjecture this is because $\BB$ has several large singular values, but many rows have small row norms $\|\BB_i\|^2\ll\tilde\ell_i$. Hence $d\gamma$ can be much smaller than $\wh{K}^2$.
	 
\end{remark}


	Note that \cref{THM:LOWNERRESULT} cannot be guaranteed by row norm sampling. Consider $\BB=\diag(\infty, 1)$. Then row norm sampling will never sample the second row, yet the leverage scores of both rows are $1$.
    All leverage scores can be computed up to a constant factor simultaneously in $\cO\p{(\nnz(\AA)+d^2)\log n}$ time. See the appendix for details.
%
%

\paragraph{Hybrid of Randomized-Deterministic Sampling.}
We propose \cref{alg:lsdetsampling} to speed up the approximation of Hessian matrices by deterministically sampling the ``heavy'' rows. The proposed method provably outperforms the vanilla leverage score sampling algorithm under a relaxed RIP condition.
\begin{restatable}{theorem}{HYBRIDUPPERBOUND}
\label{THM:HYBRIDUPPERBOUND}
	Let $\AA\in\R^{n\times d}$ and $\DD=\diag(d_1,\ldots,d_n)\in\R^{n\times n}$. For any matrix $\AA$ and any index set $N$, let $\AA_N\in\R^{n\times d}$ be such that for all $i\in N$, $(\AA_N)_i=\AA_i$, and all other rows of $\AA_N$ are $0$. 	
	Suppose $\AA^\transpose  \DD\AA = \sum_{i=1}^T\EE^i+\NN$, where
		$\EE^i=\AA_{E_i}^\transpose\DD_{E_i}\AA_{E_i}\in\R^{d\times d}$, $E_i$ is an index set with size $\cO(d)$ (that is, at each outer iteration in \cref{alg:lsdetsampling} step 3 below, we deterministically select at most $|E_i|=\cO(d)$ rows). Let $E=\cup_{i=1}^TE_i$, $N=\{1,\ldots, n\}\backslash E$, $\NN=\AA_N^\transpose  \DD_N \AA_N\in\R^{d\times d}$.
			
	Assume $\DD_N^{1/2}\AA_N$ has the following {\bf relaxed restricted isometry property (RIP) with parameter $\rho$}. That is, with probability $1-1/n$ over uniformly random sampling matrices $\bSS$ with $t=\cO(d\|\AA^\transpose \DD\AA \|/\epsilon^2)$ rows each scaled by $\sqrt{{n}/{t}}$, we have
$
	\forall \xx\not\in\text{ kernel}(\DD_N^{1/2}\AA_N), \|\bSS\DD_N^{1/2}\AA_N\xx\|=(1\pm\epsilon)\rho \|\xx\|.
$
	
	Also assume that for some constant $c>1$:
	$
	c\|\AA_N^\transpose|\DD_N|\AA_N\|\leq \|\sum_i\EE^i\|.
	$
	Then the sketch can be expressed as $\sum_i\EE^i+\AA_N^\transpose\DD_N^{1/2}\bSS^\transpose  \bSS \DD_N^{1/2} \AA_N$ and, with probability $1-\cO({1}/{d})$, we have
	$
	\|\sum_i\EE^i+\AA_N^\transpose\DD_N^{1/2}\bSS^\transpose  \bSS \DD_N^{1/2} \AA_N-\AA^\transpose  \DD\AA \|
	\leq \epsilon.
	$
\end{restatable}
\begin{remark}
	The takeaway from \cref{THM:HYBRIDUPPERBOUND} is that the total sample complexity of such a sampling scheme, i.e., \cref{alg:lsdetsampling}, is $\cO(Td+d\|\AA^\transpose \DD\AA \|/\epsilon^2)=\cO(d\|\AA^\transpose \DD\AA \|)$ for constant $\epsilon$ and $T$. On the other hand, the vanilla leverage score sampling scheme requires $\Omega(d\gamma\log d)$ rows. Notice that often $\gamma\gg \|\AA^\transpose \DD\AA \|$ because $\gamma$ involves $\|\AA^\transpose |\DD|\AA \|$. Although $T$ is a tunable parameter, we found in the experiments that $T=1$ performs well.
\end{remark}

\begin{algorithm}
	\caption{Hybrid Randomized-Deterministic Sampling (LS-Det) \label{alg:lsdetsampling}}
	\begin{algorithmic}[1]
		\STATE \textbf{Input:} $\DD^{1/2}\AA\in\C^{n\times d}$, iteration number $T$, threshold $m$, precision $\epsilon$, number $k$ of rows left 
		\STATE Set $k=n$
		\FOR{$t\in [T]$}
		\STATE Calculate the leverage scores $\{\ell_1,\ldots, \ell_k\}$ of $\DD^{1/2}\AA$
		\FOR {$i\in[k]$}
		\IF{$\ell_i\geq m$}
		\STATE Select row $i$, set $\DD^{1/2}\AA$ to be the set of remaining rows,  set $k=k-1$
		\ENDIF
		\ENDFOR
		\ENDFOR
		\STATE Sample $h=\cO(d/\epsilon^2)$ rows from the remaining rows using either their leverage score distribution, or uniformly at random and scaled by $\sqrt{\frac{n}{h}}$
		\STATE \textbf{Output:} The set of sampled and rescaled rows
	\end{algorithmic}
\end{algorithm}

\paragraph{Which Matrix to Sketch?}
We give a general rule of thumb that guides which matrix we should sample to get a better sample complexity. Recall that the Hessian matrix we try to sketch is of the form $\AA^\transpose \DD\AA $ where $\DD$ is diagonal. There are two natural candidates:

\begin{minipage}{.4\linewidth}
\begin{align}\label{eq:separatesampling}
\frac{s(\AA_i)+s((\DD\AA)_i)}{\sum_i \p{s(\AA_i)+s((\DD\AA)_i)}}
\end{align}
\end{minipage}%
\begin{minipage}{.4\linewidth}
\begin{align}\label{eq:togethersampling}
\frac{s((\DD^{1/2}\AA)_i)}{\sum_i s((\DD^{1/2}\AA)_i)}
\end{align}
\end{minipage}

\noindent for a score function $s:\R^{d}\to \R$ (e.g., leverage scores or row norms).
It turns out that sampling according to \cref{eq:separatesampling} can lead to an arbitrarily worse upper bound than sampling using \cref{eq:togethersampling}.

\begin{theorem}\label{THM:UPPERBOUNDCOMPARISON}
	Let $\AA\in\R^{n\times d}$. Let $\DD\in\R^{n\times n}$ be diagonal. Let $\TT$ be an $\epsilon$ $\ell_2$-subspace embedding for $\text{span}(\AA, \DD\AA)$ and $\bSS$ be an $\epsilon$ $\ell_2$-subspace embedding for $\text{span}(\DD^{1/2}\AA, (\DD^{1/2})^{\hermconj}\AA)$. Sampling by \cref{eq:separatesampling} can give an arbitrarily worse upper bound than sampling by \cref{eq:togethersampling}.
\end{theorem}

%
%
%
%

\subsection{Application to Optimization Algorithms}
\label{SEC:EXP}

As mentioned previously, to accelerate convergence of second-order methods with an inexact Hessian, one needs to construct the sub-sampled matrix such that $\HH(\xx)\approx\nabla^2F(\xx)$. In randomized sub-sampling of the Hessian matrix, we select the $i$-th term in $\sum_{i=1}^n\nabla^2f_i(\xx)$ with probability $p_i$, restricting $\sum_{i\in[n]}p_i=1$. Let $\cS$ denote the sample collection and define
$
\HH(\xx) \triangleq \frac{1}{n|\cS|}\sum_{i\in\cS}\frac{1}{p_i}\nabla^2f_i(\xx) + \nabla^{2} r(\xx).
$
Uniform oblivious sampling is done with $ p_{i} = 1/n $, which often results in a poor approximation unless $|\cS| \gg 1$. Leverage score sampling is in some sense an optimal data-aware sampling scheme where each $ p_{i} $ is proportional to the leverage score $ \ell_i $ (see \cref{alg:lssampling}). 

One condition on the quality of approximation $\HH(\xx)\approx\nabla^2F(\xx)$ is typically taken to be 
\begin{align} \label{eq:H}
\|\HH(\xx)-\nabla^2F(\xx)\|\leq\epsilon, \quad \text{for some } \; 0 < \epsilon \leq 1,
\end{align}
which has been considered both in the contexts of convex and non-convex Newton-type optimization methods \citep{roosta2019sub,bollapragada2019exact,xuNonconvexTheoretical2017,yao2018inexact,liu2019stability}. 
For convex settings where $ \nabla^2 F(x) \succeq \zero $, a stronger condition can be considered as 
\begin{align}\label{eq:S}
(1-\epsilon) \nabla^2F(x)  \preceq \HH(\xx) \preceq (1+\epsilon )\nabla^2F(x),
\end{align}
which, in the context of sub-sampled Newton's method, leads to a faster convergence rate than \cref{eq:H} \cite{roosta2019sub,xu2016sub,liu2017inexact}. However, in all prior work, \cref{eq:S} has only been considered in the restricted case where each $ f_{i} $ is convex. Here, using the result of \cref{SEC:OPTIMIZATION_LS}, we show that \cref{eq:S} can also be guaranteed in a more general case where the $ f_{i} $'s in \cref{eq:obj_sum} are allowed to be non-convex. We demonstrate the theoretical advantages of complex leverage score sampling in \cref{alg:lssampling,alg:lsdetsampling} as a way to guarantee \cref{eq:S,eq:H} in convex and non-convex settings, respectively. 
For the convex case, we consider sub-sampled Newton-CG \citep{roosta2019sub,xu2016sub}. For non-convex settings, we have chosen two examples of Newton-type methods: the classical trust-region \citep{conn2000trust} and the more recently introduced Newton-MR method \citep{roosta2018newton}. We emphasize that the choice of these non-convex algorithms was, to an extent, arbitrary and we could instead have picked any Newton-type method whose convergence has been previously studied under Hessian approximation models, e.g., adaptive cubic regularization \citep{yao2018inexact,tripuraneni2017stochasticcubic}. The details of these optimization methods and theoretical convergence results are deferred to the supplementary.

We verify the results of \cref{SEC:OPTIMIZATION_LS} by evaluating the empirical performance of the non-uniform sampling strategies proposed in the context of Newton-CG, Newton-MR and trust-region, see details in the appendix.

\begin{figure}[ht]
	\centering
	\subfigure[$ \|\nabla F(x_{k})\| $ vs.\ Oracle calls (\texttt{Drive Diagnostics})]
	{\includegraphics[scale=0.35]{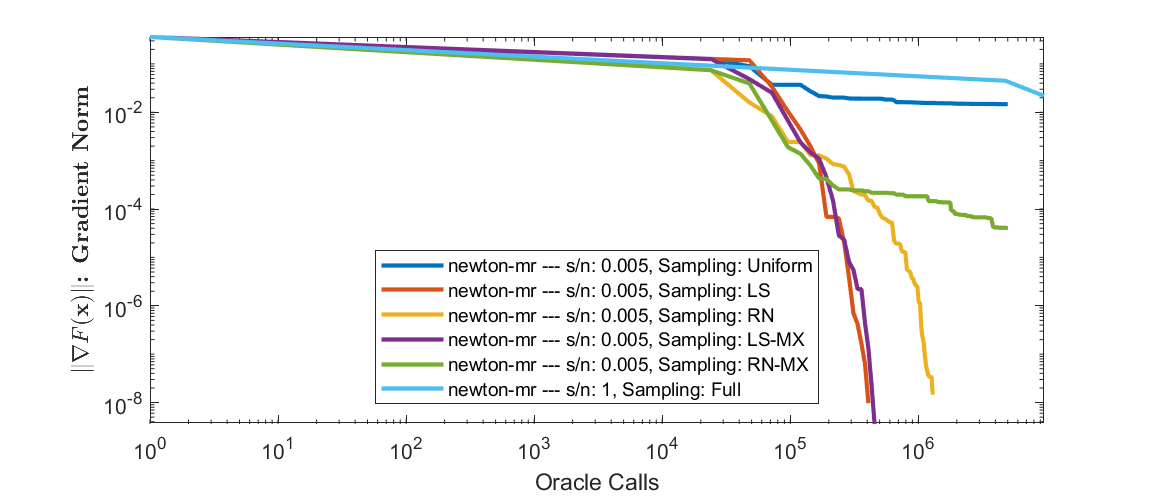}}
	\subfigure[$ \|\nabla F(x_{k})\| $ vs.\ Oracle calls (\texttt{Cover Type})]
	{\includegraphics[scale=0.35]{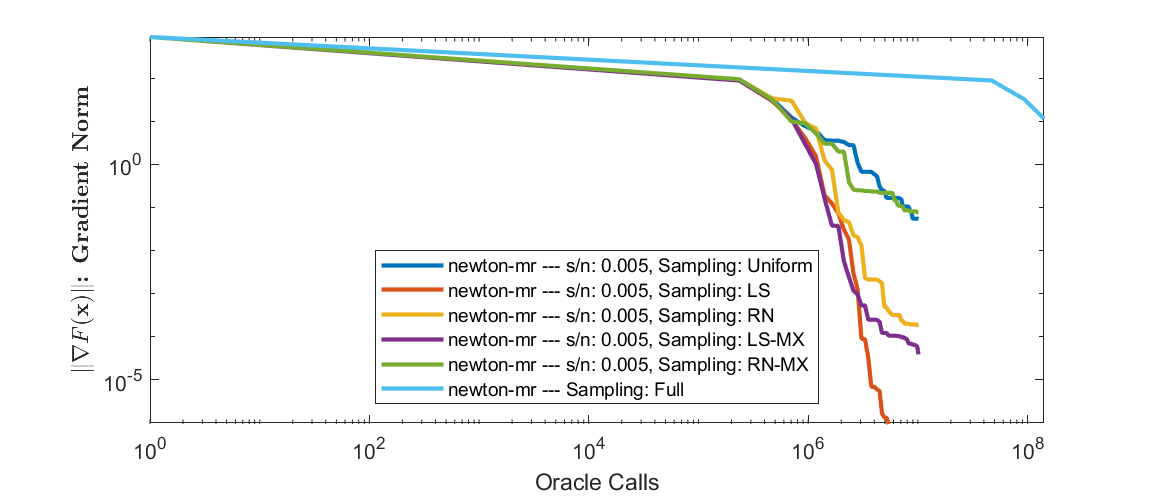}}
	\subfigure[$ \|\nabla F(x_{k})\| $ vs.\ Oracle calls (\texttt{UJIIndoorLoc})]
	{\includegraphics[scale=0.35]{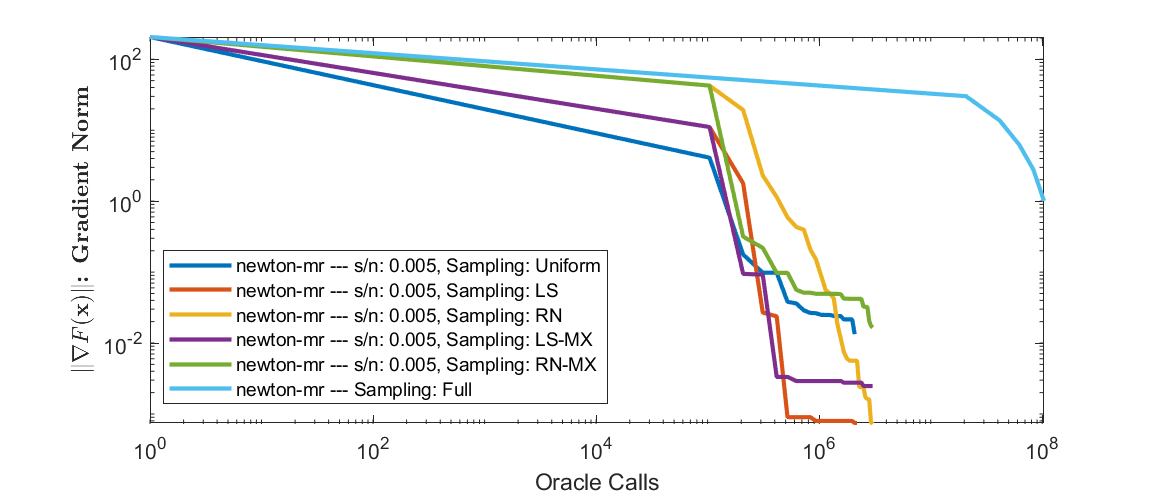}}
	\caption{Comparison of Newton-MR with various sampling schemes. \label{fig:newton_mr}}
\end{figure}

\begin{figure*}[ht]
	\centering
	\subfigure[$ F(x_{k}) $ vs.\ Oracle calls]
	{\includegraphics[scale=0.35]{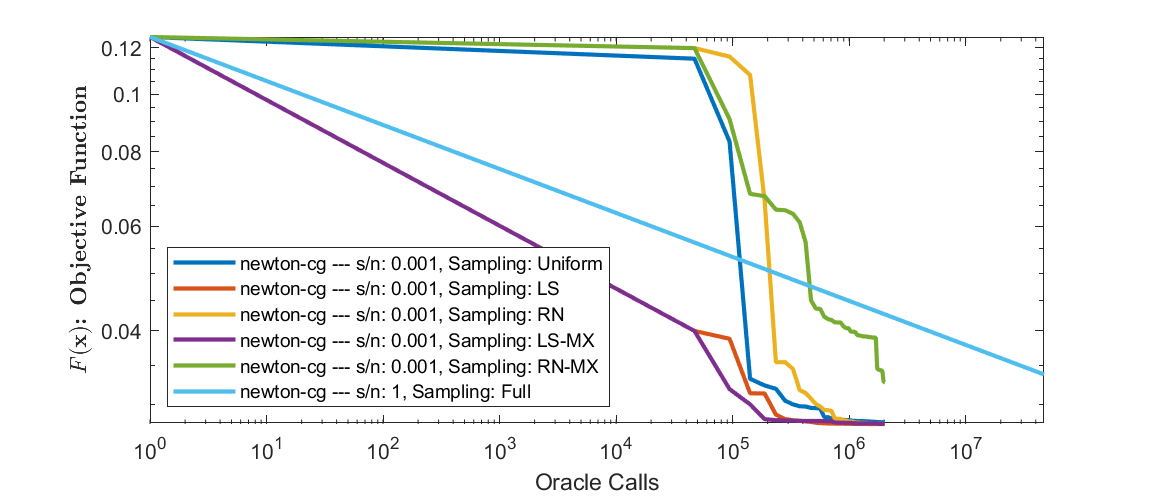}}
	\subfigure[$ \|\nabla F(x_{k})\| $ vs.\ Oracle calls]
	{\includegraphics[scale=0.35]{./figs/nlls_none/covetype/newton-mr/Grad_Props.png}} 
	\subfigure[$ F(x_{k}) $ vs.\ Oracle calls]
	{\includegraphics[scale=0.35]{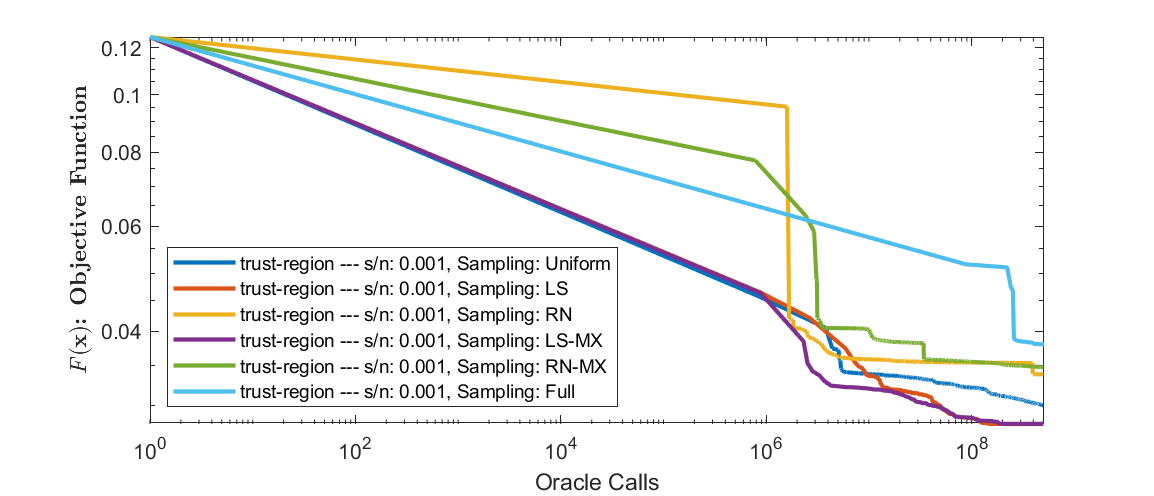}}
	\subfigure[$ F(x_{k}) $ vs.\ Oracle calls]
	{\includegraphics[scale=0.35]{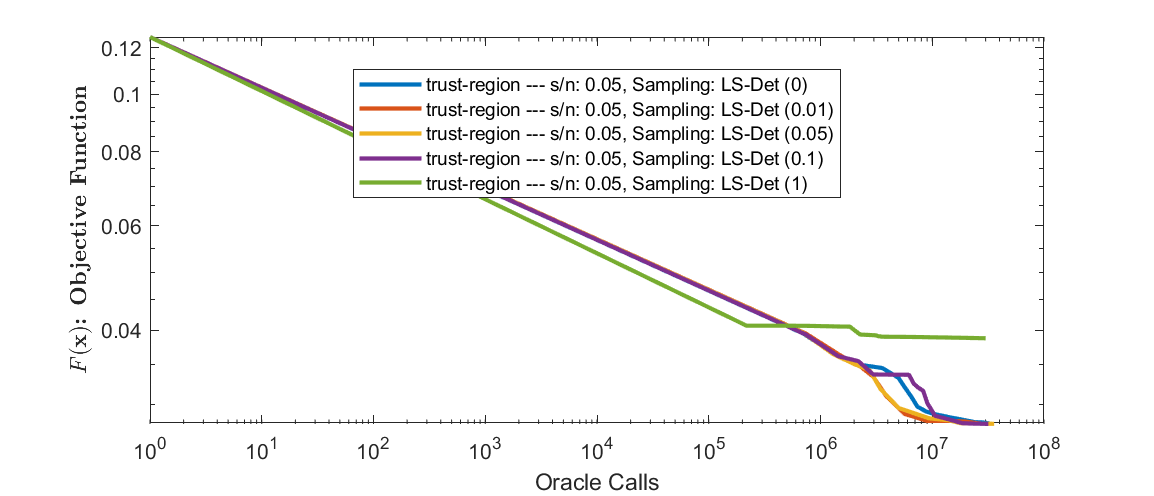}}
	\caption{Comparison of (a) Newton-CG, (b) Newton-MR, and (c) Trust-region using various sampling schemes. (d)  Performance of the hybrid sampling scheme. \label{fig:optimization}}
\end{figure*}

\noindent
\paragraph{Sub-sampling Schemes.}
We focus on several sub-sampling strategies (all are done with replacement).
\emph{Uniform}: For this we have $ p_{i} = 1/n, \; i = 1,\ldots,n $. \emph{Leverage Scores (LS)}: Complex leverage score sampling by considering the leverage scores of $ \DD^{1/2}\AA $ as in \cref{alg:lssampling}. \emph{Row Norms (RN)}: Row-norm sampling of $ \DD^{1/2} \AA $ using \cref{eq:togethersampling} where $s((\DD^{1/2} \AA)_i)=|f_{i}{''}(\dotprod{\aa_{i}, \xx})|\vnorm{\aa_i}_{2}^{2}$.
\emph{Mixed Leverage Scores (LS-MX)}: A mixed leverage score sampling strategy arising from a non-symmetric viewpoint of the product $ \AA^{\transpose} \left( \DD \AA\right) $ using \cref{eq:separatesampling} with $s(\AA_i)=\ell_i(\AA)$ and $s((\DD\AA)_i)=\ell_i(\DD\AA)$. \emph{Mixed Norm Mixture (RN-MX)}: A mixed row-norm sampling strategy with the same non-symmetric viewpoint as in \cref{eq:separatesampling} with $s(\AA_i)=\vnorm{(\AA)_i}$ and $s((\DD\AA)_i)=\vnorm{(\DD\AA)_i}$. \emph{Hybrid Randomized-Deterministic (LS-Det)}: Sampling using \cref{alg:lsdetsampling}. \emph{Full}: In this case, the exact Hessian is used.

\noindent
\paragraph{Model Problems and Datasets.}
We consider the task of binary classification using the non-linear least squares formulation of \cref{eq:obj_sum}. Numerical experiments in this section are done using \texttt{covertype}, \texttt{Drive Diagnostics}, and  \texttt{UJIIndoorLoc} from the UC Irvine ML Repository \cite{Dua:2019}.

\noindent
\paragraph{Performance Evaluation.}
For Newton-MR, the convergence is measured by the norm of the gradient, and hence we evaluate it using various sampling schemes by plotting $ \vnorm{\nabla F(\xxk)} $ vs.\ the total number of oracle calls. For Newton-CG and trust-region, which guarantee descent in objective function, we plot $ F(\xxk) $ vs.\ the total number of oracle calls. We deliberately choose not to use ``wall-clock'' time since it heavily depends on the implementation details and system specifications. 

\paragraph{Comparison Among Various Sketching Techniques.}
We present empirical evaluations of Uniform, LS, RN, LS-MX, RN-MX and Full sampling in the context of Newton-MR in \cref{fig:newton_mr}, and evaluation of all sampling schemes in Newton-CG, Newton-MR, Trust-region, and hybrid sampling on \texttt{covertype} in \cref{fig:optimization}. For all algorithms, LS and RN sampling amounts to a more efficient algorithm than that with LS-MX and RN-MX variants respectively, and at times this difference is more pronounced than other times, as predicted in \cref{THM:UPPERBOUNDCOMPARISON}. Meanwhile, LS and LS-MX often outperform RN and RN-MX, as proven in \cref{THM:LOWNERRESULT} and \cref{THM:NORMRESULT}.

\noindent
\paragraph{Evaluation of Hybrid Sketching Techniques.}
To verify the result of \cref{alg:lsdetsampling}, we evaluate the performance of the trust-region algorithm by varying the terms involved in $ \EE $, where we call the rows with large leverage scores heavy, and denote the matrix formed by the heavy rows by $\EE$. The matrix formed by the remaining rows is denoted by $\NN$, see \cref{THM:HYBRIDUPPERBOUND} for details. We do this for a simple splitting of $ \HH = \EE + \NN $. We fix the overall sample size and change the fraction of samples that are deterministically picked in $ \EE $. The results are depicted in \cref{fig:optimization}. The value in brackets after LS-Det is the fraction of samples that are included in $ \EE $, i.e., deterministic samples. ``LS-Det (0)'' and ``LS-Det (1)'' correspond to $ \EE = \bm{0}$ and $ \NN = \bm{0} $, respectively. The latter strategy has been used in low rank matrix approximations \citep{mccurdy2018ridge}. As can be seen, the hybrid sampling approach is always competitive with, and at times strictly better than, LS-Det (0). As expected, LS-Det (1), which amounts to entirely deterministic samples, consistently performs worse. This can be easily attributed to the high bias of such a deterministic estimator.

\section{Sketch-and-Solve Paradigm For Complex Regression}\label{SEC:SKETCHFORREG}
\subsection{Theoretical Results}
Recall the $\ell_p$-regression problem:
\begin{align}\label{eq:l1obj}
	\min_{\xx}\|\AA\xx-\bb\|_p
\end{align}
Here we consider the complex version: $\AA\in\C^{n\times d}, \bb\in\C^{n}, \xx\in\C^{d}$.

The inner product over the complex field can be embedded into a higher-dimensional real vector space. It suffices to consider the scalar case.
\begin{lemma}\label{lma:complexembed}
	Let $x, y\in\C$ where $x=a+bi, y=c+di$, let $\phi:\C\to\R^2$ via
	$
		\phi(x)=\phi(a+bi)=[a\ \ b].
	$
	Let $\sigma:\C\to\R^{2\times 2}$ via:
	$
		\sigma(y)=\sigma(c+di)=\begin{bmatrix}
			c & d\\-d & c
		\end{bmatrix}
	$
	. Then $\phi$ and $\sigma$ are bijections (between their domains and images), and we have 
	$
		\phi(\bar y x)^\transpose=\sigma(y)\phi(x)^\transpose
	$
\end{lemma}
\begin{proof}[Proof of \cref{lma:complexembed}]
	It is clear that $\phi, \sigma$ are bijections between their domains and images.
    $
		\sigma(y)\phi(x)^\transpose=\begin{bmatrix}
			c & d\\ -d & c
		\end{bmatrix}\begin{bmatrix}
			a\\ b
		\end{bmatrix}=\begin{bmatrix}
			ac+bd\\
			cb-ad
		\end{bmatrix}=\phi(\bar yx)^\transpose.
	$
\end{proof}
We apply $\sigma$ to each entry in $\AA$ and concatenate in the natural way. Abusing notation, we then write:
$
	\AA'=\sigma(\AA) \in\R^{2n\times 2d}.
$ 
Similarly, we write
$
	\xx'=\phi(\xx)\in\R^{2d},\ \bb'=\phi(\bb)\in\R^{2d}.
$

\begin{fact}
 For $1\leq p<q\leq\infty$ and $\xx\in\R^d$, we have
 $
 	\|\xx\|_q\leq\|\xx\|_p\leq d^{1/p-1/q}\|\xx\|_q.
 $
 \end{fact}
 \begin{fact}
	An $\ell_p$ Auerbach basis is well-conditioned and always exists. 
\end{fact}

\begin{definition}\label{lma:12norm}
	Let $\xx\in\R^{2d}$ for some $d\in\N$. Define
	$
		\qnorm{\xx}_{p, 2}=\p{\sum_{i=0}^{d-1} \p{\xx_{2i+1}^2+\xx_{2i+2}^2}^{p/2}}^{1/p}.
	$
	Note that letting $\xx\in\C^d$, we have $\|\xx\|_p=\qnorm{\phi(\xx)}_{p,2}$.
\end{definition}

By \cref{lma:12norm}, we can ``lift'' the original $\ell_p$-regression to $\R^{2d}$ and solve 
$
	 \min_{\phi(\xx)\in\R^{2d}}\qnorm{\sigma(\AA)\phi(\xx)-\phi(\bb)}_{p,2}
$ instead of \cref{eq:l1obj}. 
This equivalence allows us to consider sketching techniques on real matrices with proper modification. See \cref{alg:fastreall1} for details. In turn, such an embedding gives an arbitrarily good approximation with high probability, as shown in \cref{THM:COMPLEXSKETCH}.

\begin{algorithm}[ht]
	\caption{Fast Algorithm for Complex $\ell_p$-regression \label{alg:fastreall1}}
	\begin{algorithmic}[1]
		\STATE \textbf{Input:} $\AA'\in\R^{2n\times 2d}, \bb'\in\R^{2d}$, precision $\epsilon>0$, $p\in[1,\infty]$
		\STATE Set $t\geq C\frac{\log(1/\epsilon)\poly(d)}{\epsilon^2}$, $\gamma=d^{-1/q-1}$, $\sigma= \frac{ \sqrt\pi}{2^{p/2}\Gamma((p+1)/2)}$, $\cP=\{P=(2i+1, 2i+2)\in\R^2, \forall i\in\{0,\ldots, d-1\}\}$
		\IF {$p\neq\infty$}
		\STATE Compute the $\ell_p$ leverage score $\{\ell([\AA'_1\ \ \bb'_1]),\ldots, \ell([\AA'_n\ \ \bb'_n])\}$ for each row of $[\AA'\ \ \bb']$
		\STATE Let $\cP_h$ be the collection of $P=(a,b)\in\cP$ such that $\ell([\AA'_a\ \ \bb'_a])\geq \gamma$ or $\ell([\AA'_b\ \ \bb'_b])\geq\gamma$. Denote the collection of all other pairs $\cP_l$
		\STATE Rearrange $[\AA'\ \ \bb']$ such that the rows in $\cP_h$ are on top
		\FOR {Each $P_i\in\cP_h$}
			\STATE Sample a Gaussian random matrix $\GG_{P_i}\in\R^{t\times 2}$, where each entry follows $\cN(0, \sigma^2)$
		\ENDFOR
		\FOR {Each $Q_i\in\cP_l$}
		\STATE Sample Gaussian random matrix $\GG_{Q_i}\in\R^{1\times 2}$, where each entry follows $\cN(0, \sigma^2)$
		\ENDFOR
		\STATE Define a block diagonal matrix $\GG=\diag(\GG_{P_1},\ldots,\GG_{P_{|\cP_h|}}, \GG_{Q_1},\ldots, \GG_{Q_{|\cP_l|}})$

		\STATE \textbf{Output:} $\|\GG\AA'\yy-\GG\bb'\|_p$
		\ELSE
			\FOR{Each $P_i\in\cP$}
			\vspace{1mm}
				\STATE Sample Gaussian matrix $\GG_{P_i}\in\R^{s\times 2}$ with entries from $\cN(0, \frac{\pi}{2})$, where $s=\cO\p{\frac{\log(1/\epsilon)}{\epsilon^2}}$
				\STATE Let $\RR_{P_i}\in\R^{2^{s}\times s}$ where each row of $\RR_{P_i}$ is a vector in $\{-1, +1\}^{s}$
			\ENDFOR
			\STATE Let $\RR=\diag(\RR_1,\ldots, \RR_{|\cP|})$, $\GG=\diag(\GG_1,\ldots, \GG_{|\cP|})$
			\STATE \textbf{Output:} $\|\RR\GG\AA'\yy-\RR\GG\bb'\|_\infty$
		\ENDIF
	\end{algorithmic}
\end{algorithm}

\begin{theorem}\label{THM:COMPLEXSKETCH}
	Let $\AA\in\C^{n\times d}, \bb\in\C^n$. Then \cref{alg:fastreall1} with input $\AA':=\sigma(\AA)$ and $\bb':=\phi(\bb)$ returns a regression instance whose optimizer is an $\epsilon$ approximation to \cref{eq:l1obj}, with probability at least $0.98$. The total time complexity for $p\in[1,\infty)$ is $\cO(\nnz(\AA)+\poly(d/\epsilon))$; for $p=\infty$ it is $\cO(2^{\tilde\cO(1/\epsilon^2)}\nnz(\AA))$. The returned instance can then be optimized by any $\ell_p$-regression solver.
\end{theorem}

\begin{proof}
	For simplicity, we let $\AA\in\R^{2n\times 2d}$ to avoid repeatedly writing the prime symbol in $\AA'$.

Let $\yy\in\R^{2n}$ be arbitrary. Let $P, \cP, \cP_l, \cP_h$ be as defined in \cref{alg:fastreall1}.
We say a pair $P=(a,b)$ is heavy if $P\in\cP_h$, and light otherwise.

As an overview, when $p\in[1,\infty)$, for the heavy pairs, we use a large Gaussian matrix and apply Dvoretsky's theorem to show the $\qnorm{\cdot}_{p,2}$ norm is preserved. For the light pairs, we use a single Gaussian vector and use Bernstein's concentration. This is intuitive since the heavy pairs represent the important directions in $\AA$, and hence we need more Gaussian vectors to preserve their norms more accurately; but the light pairs are less important and so the variance of the light pairs can be averaged across multiple coordinates. Hence, using one Gaussian vector suffices for each light pair. For $p=\infty$, we need to preserve the $\ell_2$ norm of every pair, and so in this case we apply Dvoretsky's theorem  to sketch every single pair in $\cP$.

We split the analysis into two cases: $p\in[1,\infty)$ and $p=\infty$. In the main text we only present $p\in[1,\infty)$ and defer the other case to the appendix.

\paragraph{Case 1: $p\in[1,\infty)$.}
	
		\textsl{For light rows}:

Let $\UU$ be an Auerbach basis of $\AA$ and $\yy=\UU\xx$, where $\xx\in\R^d$ is arbitrary. Then for any row index $i\in[n]$:
\begin{align*}
	|\yy_i|=|\UU_i\xx|\leq \|\UU_i\|_p\|\xx\|_q \leq  d^{1/q}\|\UU_i\|_p\|\yy\|_p 
	\Longrightarrow  \frac{d^{-1/q}|\yy_i|}{\|\yy\|_p}\leq \|\UU_i\|_p
\end{align*}
where the first step is because the Auerbach basis satisfies $\|\UU\xx\|_p=\|\yy\|_p\geq d^{-1/q} \|\xx\|_q$. This implies that if $|\yy_i|\geq\gamma d^{1/q} \|\yy\|_p$, then $ \|\UU_i\|_p\geq\gamma$.
Hence, by definition of $\gamma$, if $\|\UU_i\|_p\leq d^{-1/q-1}$, then $|\yy_i|\leq d^{-1}\|\yy\|_p$. 

For any light pair $P=(a,b)$, we sample two i.i.d. Gaussians $\bgg_P=(\bgg_a, \bgg_b)$ from $\cN(0, \sigma^2)$ where $\sigma= \frac{ \sqrt\pi}{2^{p/2}\Gamma((p+1)/2)}$. Since 
$\bgg_a\yy_a+\bgg_b\yy_b\sim\cN(0,\|\yy_P\|_2^2\sigma^2)$, we have $\E[|\bgg_a\yy_a+\bgg_b\yy_b|^p] =  \|\yy_P\|_2^p$.

Let $\cP_l$ be the set of all light pairs. We have,
$
	\E[\sum_{P\in\cP_l}|\ip{\bgg_P}{\yy_P}|^p]=\sum_{P\in\cP_l}\|\yy_P\|_2^P.
$

Let random variable $Z_P = |\ip{\bgg_P}{\yy_P}|$. This is $\sigma^2\|\yy_P\|_2^2$-sub-Gaussian (the parameter here can be improved by subtracting $\mu_{Z_P}^2$). 

Define event $\cA$ to be :
$\max_{P\in\cP_l} Z_P^p \geq \cO((\log(|\cP_l|)+\poly(d))^p)\triangleq t$.
Since the $Z_P$ variables are sub-Gaussian and $p\geq 1$, we have that $(\cdot)^p$ is monotonically increasing, so we can use the standard sub-Gaussian bound to get
\begin{align}\label{eq:maxgauss}
	P(\cA)< \exp\p{-\poly(d)}.
\end{align}
Note that $p$ is a constant, justifying the above derivation.

Also note that $\var(Z_P)=\cO(\|\yy_P\|^{2p}_2)$. Condition on $\neg \cA$. By Bernstein's inequality, we have:
\begin{align}\label{eq:lightrowconcentration}
	\begin{split}
	&P\p{\abs{\sum_{P\in\cP_l}Z_P^p - \E[|\sum_{P\in\cP_l}Z_P^p| ]}>\epsilon \E[|\sum_{P\in\cP}Z_P^p|] }\\
	=&P\p{\abs{\sum_{P\in\cP_l}Z_P^p - \sum_{P\in\cP_l}\|\yy_P\|_2^p}>\epsilon \sum_{P\in\cP}\|\yy_P\|_2^p }\\
	\leq & \exp\p{-\frac{\epsilon^2(\sum_{P\in\cP}\|\yy_P\|_2^p)^2}{\sum_{P\in\cP_l} \|\yy_P\|^{2p}_2 +\epsilon t\sum_{P\in\cP_l}\|\yy_P\|_2^p}}\\
	\leq & \exp\p{-\cO\p{\frac{\epsilon^2(\sum_{P\in\cP}\|\yy_P\|_2^p)^2}{\sum_{P\in\cP_l} \|\yy_P\|^{2p}_2 }}}\\
	\leq &\exp\p{-\cO(\epsilon^2/\gamma)}=\exp\p{\cO(-\epsilon^2d\log d)}
	\end{split}
\end{align}
where the last step follows from the definition of light pairs. Using that if for all $P=(a,b)\in\cP_l$, $|\yy_a|,|\yy_b|<\cO(d^{-1}) \|\yy\|_p $, then $\|\yy_P\|_2^p<\cO(d^{-1})\|\yy\|_p^p$, we have
\[
	\sum_{P\in\cP_l}\|\yy_P\|_2^{2p}\leq \frac{1}{\cO(d^{-1})}\cO(d^{-2})\|\yy\|_p^{2p}=\cO(d^{-1})\|\yy\|_p^{2p}.
\]
Since $\sum_{P\in\cP}\|\yy_P\|_2^p\leq \cO(\|\yy\|_p^p)$, we will have $\frac{(\sum_{P\in\cP}\|\yy_P\|_2^p)^2}{\sum_{P\in\cP_l} \|\yy_P\|^{2p}_2 }=\cO(d)$.

\emph{Net argument.} In the above derivation, we fix a vector $\yy\in\R^{2n}$. Hence for the above argument to hold for all pairs, a na\"ive argument will not work since there are an infinite number of pairs. However, note that each pair lives in a two-dimensional subspace. Hence, we can take a finer union bound over $\cO((1+\epsilon)^2/\epsilon^2)$ items in a two-dimensional $\ell_2$ space using a net argument. This argument is standard, see, e.g., \cite[Chapter 2]{woodruff2014sketching}.

Using the net argument, \cref{eq:lightrowconcentration,eq:maxgauss} holds with probability at least $1-\cO(e^{-d})$ for all $\yy\in\R^{2n}$ simultaneously. In particular, with probability at least $0.99$:
\begin{align}\label{eq:lightrowfinal}
	\sum_{P\in\cP_l}|\ip{\bgg_P}{\yy_P}|^p-\sum_{P\in\cP_l} \|\yy_P\|_2^p \in (\pm\epsilon)\sum_{P\in\cP} \|\yy_P\|_2^p.
\end{align}

		\textsl{For heavy rows:}
		 
			Since the $\ell_p$ leverage scores sum to $d$, there can be at most $d/\gamma=\poly(d)$ heavy rows. For each pair $P\in\cP_h$, we construct a Gaussian matrix $\GG_P\in\R^{s\times 2}$, where $s=\poly(d/\epsilon)$. Applying Dvoretsky's theorem for $\ell_p$ (\cite{paouris2017random} Theorem 1.2), with probability at least $0.99$, for all $2$-dimensional vectors $\yy_P$, $\|\GG_P\yy_P\|_p=(1\pm\epsilon)\|\yy_P\|_2$. Hence,  
			\begin{align}\label{eq:heavyrowfinal}
				\sum_{P\in\cP_h}\|\GG_P\yy_P\|_p^p=(1\pm\Theta(\epsilon))\sum_{P\in\cP_h}\|\yy_P\|_2^p.
			\end{align}
		Combining \cref{eq:lightrowfinal} and \cref{eq:heavyrowfinal}, we have with probability at least $0.98$, for all $\yy\in\R^{2n}$:
		\[
			\|\GG\yy\|_p^p=(1\pm\Theta(\epsilon)) \sum_{P\in\cP}\|\yy_P\|_2^p = (1\pm\Theta(\epsilon))\qnorm{\yy}_{p,2}^p
		\]
		Letting $\yy=\UU\xx-\bb$ and taking the $1/p$-th root, we obtain the final claim.

\paragraph{Case 2: $p=\infty$.}

	In this case, we first construct a sketch $\GG$ to embed every pair into $\ell_1$. That is, for all $P\in\cP$, construct a Gaussian matrix $\GG_p\in\R^{\cO\p{\frac{\log(1/\epsilon)}{\epsilon^2}}\times 2}$. By Dvoretsky's theorem, with probability at least $0.99$, for all $\yy_P\in\R^2$, we have $\|\GG_P\yy_P\|_1 = (1\pm\epsilon)\|\yy_P\|_{2}$. Hence
	\begin{align}\label{eq:linftytol1}
	\begin{split}
		\text{For all $\yy\in\R^{2n}$ and any $P\in\cP$: }
		 \max_{P\in\cP}\|\GG_P\yy_P\|_1 = \max_{P\in\cP} (1\pm\epsilon)\|\yy_P\|_{2}=\qnorm{\yy}_{\infty, 2}
	\end{split}
	\end{align}
	However, we do not want to optimize the left hand side directly.	
	
	Recall that by construction $\GG=\diag(\GG_1,\ldots,\GG_{|\cP|})$ is a block diagonal matrix. 
	
	Construct $\RR$ as in \cref{alg:fastreall1}. By \cite{indyk2001algorithmic}, for all $P\in\cP$, $\RR_P$ is an isometric embedding $\ell_1\hookrightarrow\ell_\infty$, i.e., for all $\yy_P\in\R^{2}$: 
	$
		\|\RR_P\GG_P\yy_P\|_\infty = \|\GG_P\yy_P\|_1.
	$
	
	Combining this with \cref{eq:linftytol1}, we get that with probability at least $0.99$ for all $\yy\in\R^{2n}$:
	\[
		\|\RR\GG\yy\|_\infty = (1\pm\Theta(\epsilon))\qnorm{\yy}_{\infty,2}.
	\]
	Letting $\yy=\UU\xx-\bb$, we obtain the final claim.

\paragraph{Running time.} 
For $p\in[1, \infty)$, calculating a well-conditioned basis takes $\cO(\nnz(\AA)+\poly(d/\epsilon))$ time. Since $\GG$ is a block diagonal matrix and $\AA$ is sparse, computing $\GG\AA$ takes $\cO(\nnz(\AA)+\poly(d/\epsilon))$ time. Calculating $\GG\bb$ takes $n+\poly(d/\epsilon)$ time. Minimizing $\|\GG\AA\xx-\GG\bb\|$ up to a $(1 + \epsilon)$ factor takes $\cO(\nnz(\AA)+\poly(d/\epsilon))$ time. Using the fact that $n<\nnz(\AA)$, the total running time is $\cO(\nnz(\AA)+\poly(d/\epsilon))$.

For the case of $p=\infty$, note that $\RR$ is also a block diagonal matrix, so $\RR\GG$ can be computed by multiplying the corresponding blocks, which amounts to $\cO(2^{\tilde\cO(1/\epsilon^2)}n\frac{\log(1/\epsilon)}{\epsilon^2})$ time. $\AA$ is sparse and $\RR\GG$ is a block matrix so computing $\RR\GG\AA$ takes another $\cO(2^{\tilde\cO(1/\epsilon^2)}\nnz(\AA))$ time. Computing $\RR\GG\bb$ takes $\cO(2^{\tilde\cO(1/\epsilon^2)}n\frac{\log(1/\epsilon)}{\epsilon^2})$ time. Since $n<\nnz(\AA)$, in total these take $\cO(2^{\tilde\cO(1/\epsilon^2)}\nnz(\AA))$ time. This concludes the proof.
\end{proof}

\begin{remark}
	\cref{lma:12norm} shows for sketching complex vectors in the $\ell_p$ norm, all one needs is an embedding $\ell_2\hookrightarrow\ell_p$. In particular, for complex $\ell_2$-regression, the identity map is such an embedding with no distortion. Hence, complex $\ell_2$-regression can be sketched exactly as for real-valued $\ell_2$-regression, while for other complex $\ell_p$-regression the transformation is non-trivial.
\end{remark}

\subsection{Numerical Evaluation}
We evaluate the performance of our proposed embedding for $\ell_1$ and $\ell_\infty$ regression on synthetic data. With $\AA\in\C^{100\times 50}, \bb\in\C^{100}$, we solve $\min_{\xx\in\C^{50}}\|\AA\xx-\bb\|_1$ or $\min_{\xx\in\C^{50}}\|\AA\xx-\bb\|_\infty$. Each entry of $\AA$ and $\bb$ is sampled from a standard normal distribution (the real and imaginary coefficients are sampled according to this distribution independently). Instead of picking the heavy ($\mathcal{P}_h$) and light ($\mathcal{P}_l$) pairs , we construct a $t\times 2$ (or $s\times 2$ if $p=\infty$) Gaussian matrix for each pair (that is, we treat all pairs as heavy), as it turns out in the experiments that very small $t$ or $s$ is sufficient. For $\ell_1$ complex regression, we test with $t=2,4,6,8,10,20$, and the result is shown in \cref{fig:l1reg}. For $\ell_\infty$ complex regression, we test with $s=2,3,4,5,6$ as shown in \cref{fig:linfreg}. In both figures, the $x$-axis represents our choice of $t$ or $s$, and the $y$-axis is the approximation error $\|\hat \xx-\xx^*\|_2$, where $\hat\xx$ is the minimizer of our sketched regression problem.

\begin{figure*}
	\centering
	\subfigure[$\ell_1$ regression]
	{\includegraphics[scale=0.4]{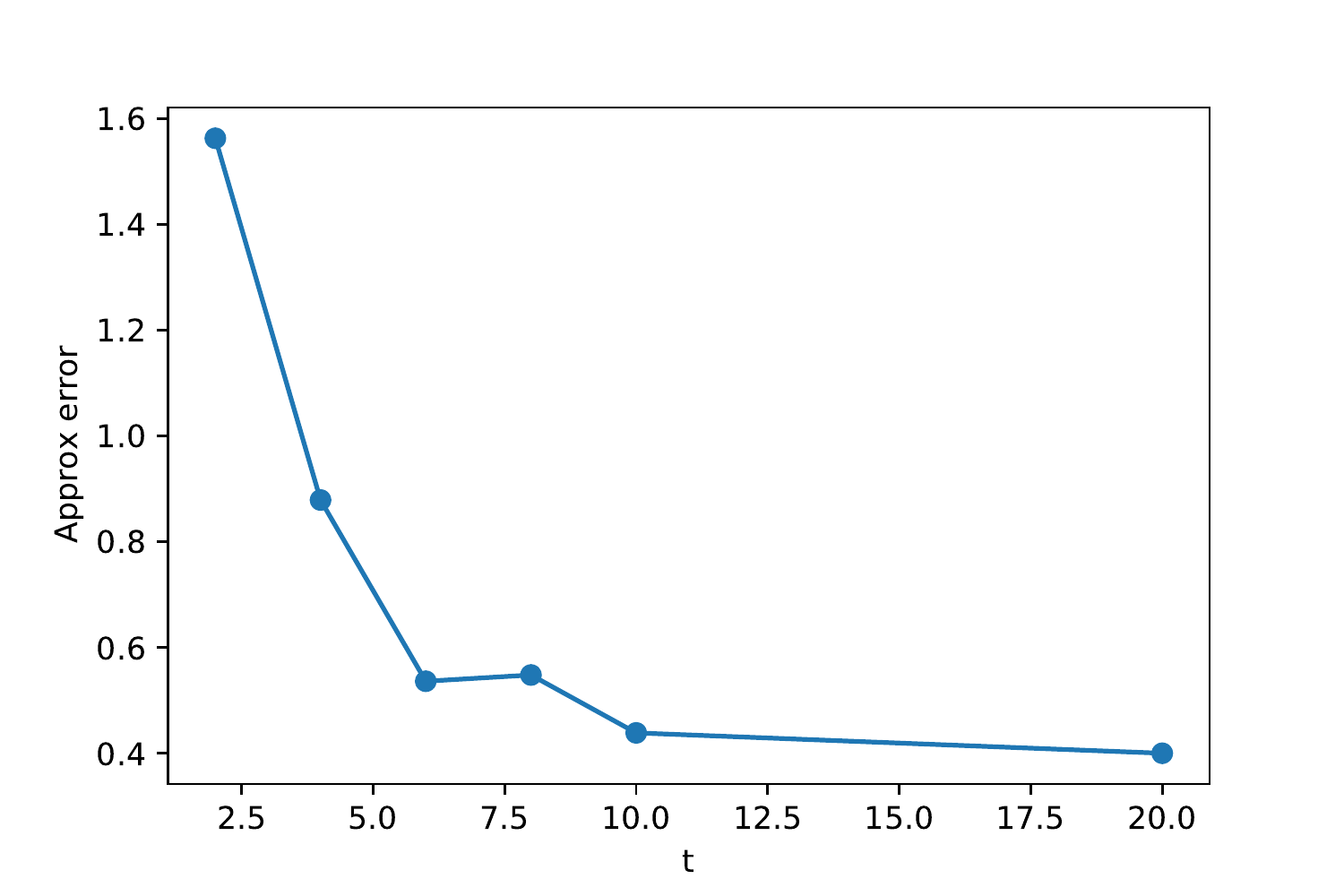} \label{fig:l1reg}}
	\subfigure[$\ell_\infty$ regression]
	{\includegraphics[scale=0.4]{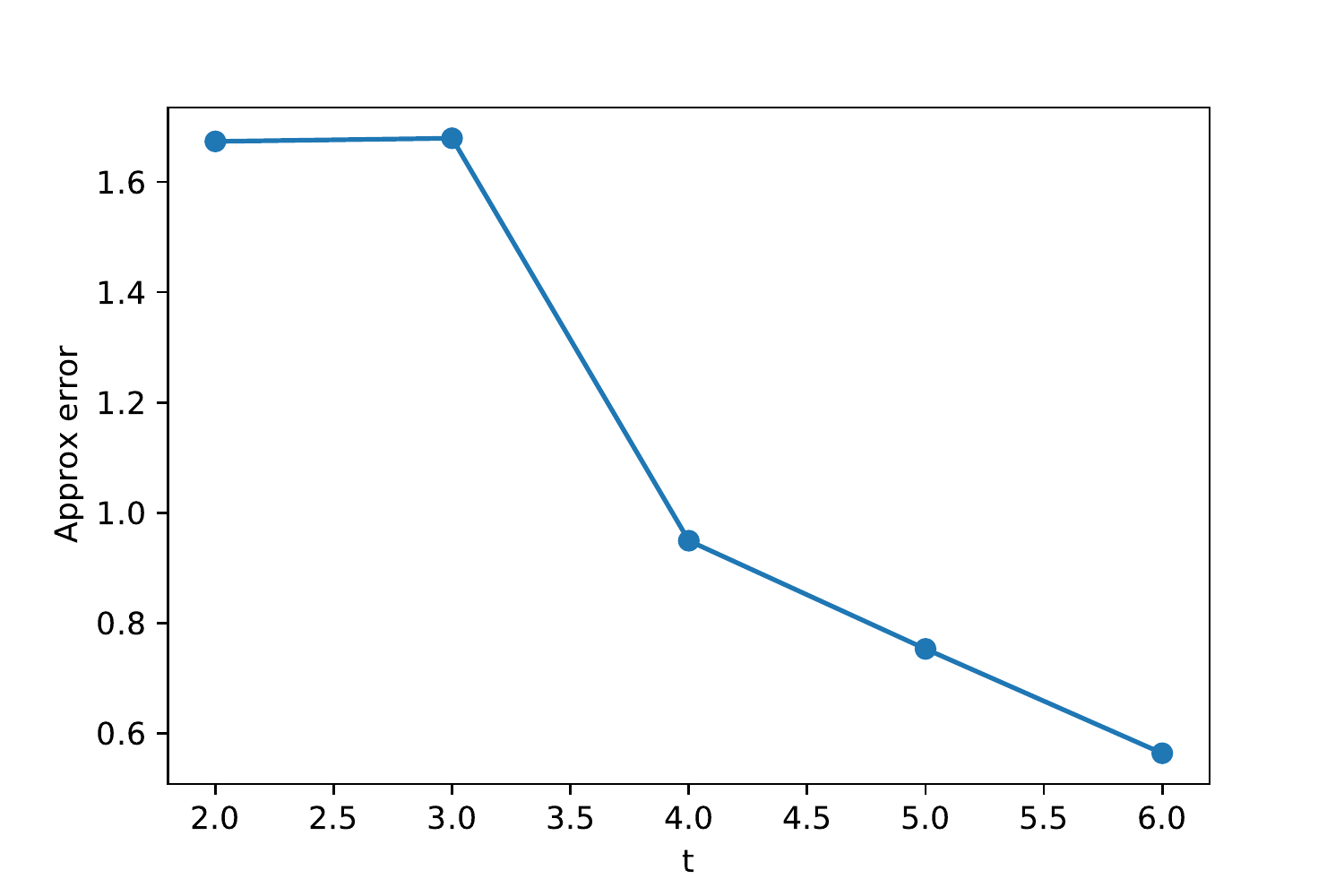} \label{fig:linfreg}} 
	\caption{Approximation error of sketched $\ell_p$ regression with complex entries.}
\end{figure*}

\section{Sketching Vector-Matrix-Vector Queries}\label{SEC:TENSORSKETCH}
The sketches in \cref{SEC:OPTIMIZATION} can also be used for vector-matrix-vector queries, but they are sub-optimal when there are a lot of cancellations. For example, if we have $\AA^\top\DD\AA = \sum_{i=1}^n d_i\aa_i\aa_i^\top$ where $\aa_1=\aa_2=\cdots=\aa_n$, $d_1=d_2=\cdots d_{n/2}$, and $d_{n/2+1}=d_{n/2+2}=\cdots=d_n$, then $\AA^\top\DD\AA=0$, yet our sampling techniques need their number of rows to scale with $\|\AA^\top|\DD|\AA\|_F$, which can be arbitrarily large. In this section, we give a sketching technique for vector-matrix-vector product queries that scales with $\|\AA^\top\DD\AA\|_F$ instead of $\|\AA^\top|\DD|\AA\|_F$. Therefore, for vector-matrix-vector product queries, this new technique works well, even if the matrices are complex. Such queries are widely used, including standard graph queries and independent set queries \citep{rashtchian2020vector}. 

 In particular, we consider a vector-matrix-vector product query $\uu^\top\MM\vv$, where $\MM=\AA^\top\BB=\sum_{i=1}^n\aa_i\bb_i^\top$ has a tensor product form, $\aa_i,\bb_i\in\C^d$, for all $i\in[n]$. One has to either sketch $\MM$ or compute $\MM$ first. Then the queries $\uu$ and $\vv$ arrive \citep{andoni2016sketching}.  In reality, this may be due to the fact that $n\gg d$ and one cannot afford to store $\AA$ and $\BB$. Our approach is interesting when $\MM$ is non-PSD and $\AA, \BB$ might be complex. This can indeed happen, for example, in a graph Laplacian with negative weights \citep{chen2016definiteness}.

\begin{algorithm}[ht]
	\caption{Tensor Sketch For Vector-Matrix-Vector Products\label{alg:tensorsketch}}
	\begin{algorithmic}[1]
		\vspace{1mm}
		\STATE \textbf{Input:} $\{\aa_i\}_{i=1}^n, \{\bb_i\}_{i=1}^n\subseteq\C^{ d}, \uu, \vv\in\C^{d}$
		\vspace{1mm}
		\STATE Let $\bSS:\C^{d}\otimes \C^{d}\to \C^{k}$ be a TensorSketch \citep{pham2013fast} with $k$ hash buckets.
		\vspace{1mm}
		\STATE Compute $ \qq = \sum_{i=1}^n\bSS(\aa_i \otimes \bb_i)\in\C^k$.
		\vspace{1mm}
		\STATE Compute $ \pp =\bSS(\uu \otimes \vv)\in\C^k$.
		\vspace{1mm}
		\STATE \textbf{Output:} $\ip{\pp}{\qq}$
	\end{algorithmic}
\end{algorithm}

\begin{theorem}
    With probability at least $0.99$, for given input vectors $\uu, \vv\in\C^d$, $\AA,\BB\in\C^{n\times d}$, \Cref{alg:tensorsketch} returns an answer $z$ such that $\abs{\zz-\uu^\top\AA^\top\BB\vv}\leq \epsilon$ in time $\tilde\cO\p{\nnz(\AA)+\frac{n\|\vv\|_2^2\|\uu\|_2^2\|\AA^\top\BB\|_F^2}{\epsilon^2}}$.
\end{theorem}
\begin{proof}
    It is known that TensorSketches are unbiased, that is,
    \[
        \E[\ip{\bSS(\AA^\top\BB)}{\bSS(\uu\otimes \vv)}]=\uu^\top\AA^\top\BB\vv,
    \]
    where $\bSS(\AA^\top\BB)=\sum_{i=1}^n\bSS(\aa_i, \bb_i)$ follows from the linearity of TensorSketch \cite{pham2013fast}. The variance is bounded by
    \[
        \operatorname{Var}\p{\ip{\bSS(\AA^\top\BB)}{\bSS(\uu\otimes \vv)}}\leq\cO\p{\frac{1}{k}\|\vv\|_2^2\|\uu\|_2^2\|\AA^\top\BB\|_F^2},
    \]
    see, e.g., \citet{pham2013fast}.
    By Chebyshev's inequality, setting $k=\cO\p{\frac{\|\vv\|_2^2\|\uu\|_2^2\|\AA^\top\BB\|_F^2}{\epsilon^2}}$ produces an estimate of $\uu^\top\AA^\top\BB\vv$ with an additive error $\epsilon$.
    
    Note that computing the sketches $\bSS(\AA^\top\BB)$ takes time $\nnz(\AA)+\nnz(\BB)+nk\log k$, and computing the inner product $\ip{\bSS(\AA^\top\BB)}{\bSS(\uu\otimes \vv)}$ takes only $k$ time. As a comparison, computing $\uu^\top\AA^\top \BB\vv$ na\"ively takes $nd^2+O(d^2)$ time, which can be arbitrarily worse than our sketched version. Note that it is prohibitive in our setting to compute $\uu^\top\AA^\top$ and $\BB\vv$ separately.
\end{proof}

\section{Conclusion}
Our work highlights the many places where non-PSD matrices and their ``square roots", which are complex matrices, arise in optimization and randomized numerical linear algebra. We give novel dimensionality reduction methods for such matrices in optimization, the sketch-and-solve paradigm, and for vector-matrix-vector queries. These methods can be used for approximating indefinite Hessian matrices, which constitute a major bottleneck for second-order optimization. We also propose a hybrid sampling method for matrices that satisfy a relaxed RIP condition. We verify these numerically using Newton-CG, trust region, and Newton-MR algorithms. We also show how to reduce complex $\ell_p$-regression to real $\ell_p$-regression in a black box way using random linear embeddings, showing that the many sketching techniques developed for real matrices can be applied to complex matrices as well. In addition, we also present how to efficiently sketch complex matrices for vector-matrix-vector queries.

{\bf Acknowledgments:} The authors would like to thank partial support from NSF grant No. CCF-181584, Office of Naval Research (ONR) grant N00014-18-1-256, and a Simons Investigator Award.

\bibliography{references}

\newpage

\begin{appendices}
\onecolumn
\section{Algorithms}
\label{SEC:APPENDIX:ALG}

\begin{algorithm}[H]
	\caption{Newton-CG With Inexact Hessian \label{alg:newton_cg}}
	\begin{algorithmic}
		\vspace{1mm}
		\STATE \textbf{Input:} Starting point $\xxo$, line-search parameter $ 0 < \rho < 1 $
		\vspace{1mm}
		\FOR {$ k = 0,1,2, \ldots $ until convergence} 
		\vspace{1mm}
		\STATE (approximately) Solve the following sub-problem using CG
		\begin{align*}
            \HHk \pp = -\bggk
        \end{align*}
		\vspace{1mm}
		\STATE Find $ \alphak $ such that 
		\begin{align*}
    		F(\xxkk) \leq F(\xxk) + \rho \alpha_k \dotprod{\ppk,  \bggk}
		\end{align*}
		\vspace{1mm}
		\STATE Update $ \xx_{k+1}  =  \xxk + \alphak \ppk $
		\vspace{1mm}
		\ENDFOR
		\vspace{1mm}
		\STATE \textbf{Output:} $ \xxk $
	\end{algorithmic}
\end{algorithm}

\begin{algorithm}[H]
	\caption{Newton-MR With Inexact Hessian \label{alg:newton_mr}}
	\begin{algorithmic}[1]
		\vspace{1mm}
		\STATE \textbf{Input:} Starting point $\xxo$, line-search parameter $ 0 < \rho < 1 $
		\vspace{1mm}
		\FOR {$ k = 0,1,2, \ldots $ until convergence} 
		\vspace{1mm}
		\STATE (approximately) Solve the following sub-problem 
		\begin{align*}
            \min_{\pp \in \real^{d}} ~~ \|\pp\| \;\; \text{subject to} \;\; \pp \in \argmin_{\widehat{\pp} \in \mathbb{R}^{d}} \vnorm{\HHk \widehat{\pp} + \bggk}.
        \end{align*}
		\vspace{1mm}
		\STATE Find $ \alphak $ such that 
		\begin{align*}
    		\vnorm{\bggkk}^{2} \leq \vnorm{\bggk}^{2} + 2 \rho \alpha_k \dotprod{\ppk, \HHk \bggk}
		\end{align*}
		\vspace{1mm}
		\STATE Update $ \xx_{k+1}  =  \xxk + \alphak \ppk $
		\vspace{1mm}
		\ENDFOR
		\vspace{1mm}
		\STATE \textbf{Output:} $ \xxk $
	\end{algorithmic}
\end{algorithm}

\begin{algorithm}[H]
	\caption{Trust Region with Inexact Hessian}
	\label{alg:tr}
	\begin{algorithmic}[1]
		\vspace{1mm}
		\STATE {\bf Input:} Starting point $\xxo$, initial radius $0 < \Delta_0 < \infty$, hyper-parameters $\eta\in(0,1), \gamma > 1$ 
		\vspace{1mm}
		\FOR{$ k = 0,1,\ldots $}
		\vspace{1mm}
		\STATE Set the approximate Hessian, $\HHk$, as in \eqref{eq:H} \label{step:STR_step}
		\vspace{1mm}
		\IF{converged}  
		\vspace{1mm}
		\STATE  Return $\xxk$.
		\vspace{1mm}
		\ENDIF
		\vspace{1mm}
		\STATE (approximately) solve the following sub-problem
		\begin{align*}
            \hspace{-1.8mm}\min_{\|\pp\|\le \Delta_{k}} m_k(\pp) \triangleq \dotprod{\nabla F(\xxk), \pp} + \frac{1}{2}\dotprod{\pp, \HHk \pp},
        \end{align*}
		\vspace{1mm}
		\STATE Set $\rho_k \triangleq \dfrac{F(\xxk + \ppk) - F(\xxk)}{m_t(\ppk)}$
		\vspace{1mm}
		\IF {$\rho_{k} \ge \eta$}
		\vspace{1mm}
		\STATE $\xxkk = \xxk + \ppk$ and $\Delta_{k+1} = \gamma \Delta_{k}$
		\vspace{1mm}
		\ELSE
		\vspace{1mm}
		\STATE $\xxkk = \xxk$ and $\Delta_{k+1} = \Delta_{k}/\gamma$
		\vspace{1mm}
		\ENDIF
		\vspace{1mm}
		\ENDFOR
		\vspace{1mm}
		\STATE {\bf Output:} $\xxk$
	\end{algorithmic}
\end{algorithm}

\section{Omitted Proof in \cref{SEC:OPTIMIZATION}}
\label{SEC:APPENDIX:PROOFS}
\subsection{Proof of \cref{THM:LOWNERRESULT}}
This theorem and proof mimic Theorem 5 in \citet{cohen2017input}. 

The statistical leverage score of the $i^{th}$ row of $\BB\in\C^{n\times d}$ can also be written as the following:
\[
	\ell_i = \BB_i(\BB^*\BB)^\dagger \BB_i^*.
\]

\begin{proof}
	Let $\BB^*=\UU\bSig\VV^*$ be the SVD of $\BB^*$. We have $\ell_i = \BB_i^*(\UU\bSig^{-2}\UU^*)\BB_i$.
	
	Let $Y=\bSig^{-1}\UU^* \p{\CC^\transpose \CC - \BB^\transpose  \BB }\UU\bSig^{-1}$. Then we write
	\[
		\YY=\sum_{j=1}^t \p{\bSig^{-1}\UU^*\p{\CC_j^\transpose \CC_j -\frac{1}{t}\BB^\transpose  \BB }\UU\bSig^{-1}}\triangleq\sum_{j=1}^t \XX_j
	\]
	where $\CC_j$ is the $j^{th}$ row of $\CC$. Note with probability $p_i$
	\[
		\XX_j = \frac{1}{t}\bSig^{-1}\UU^*\p{\frac{1}{p_i}\BB_i^\transpose \BB_i -\BB^\transpose  \BB }\UU\bSig^{-1}.
	\]
	Since $\E[\frac{1}{p_i}\BB_i^\transpose \BB_i -\BB^\transpose  \BB ]=0$ we have $\E[\YY]=0$. Also we have $\CC^\transpose \CC =\UU\bSig \YY\bSig \UU^*+\BB^\transpose  \BB $. Because $\UU\bSig^2\UU^* = \BB^*\BB $ it suffices to show that $\|\YY\|\leq \epsilon$, which gives $-\epsilon \eye\preceq \YY\preceq\epsilon \eye$, and consequently:
	\[
		\BB^\transpose \BB -\epsilon \BB^*\BB  \preceq \CC^\transpose \CC \preceq \BB^\transpose  \BB +\epsilon \BB^*\BB. 
	\]
	
	A useful tool for proving $\|\YY\|$ is small is the matrix Bernstein inequality \cite{tropp2015introduction}. We remark that the version we use is suitable for complex matrices as well.
	
	Note that for any $i$, because $\BB^\transpose \BB$ has real entries, we have
	\[
		\frac{1}{\ell_i} \BB_i^\transpose \BB_i \preceq \frac{1}{\ell_i} \BB_i^* \BB_i\preceq \BB^*\BB,
	\]
	where the first step is by the structure of $A$, and the second step follows from a known property of leverage scores (see the  proof of Lemma 4 in \citet{cohen2015uniform}). With this we have:
	\[
		\frac{1}{\ell_i}\bSig^{-1} \UU^* \BB_i^\transpose \BB_i  \UU\bSig^{-1}\preceq \bSig^{-1} \UU^*\p{\BB^*\BB }\UU\bSig^{-1}=\eye.
	\]
	Hence
	\begin{align*}
	\XX_j+\frac{1}{t}\bSig^{-1}\UU^*\BB^\transpose  \BB \UU\bSig^{-1}
		\preceq \frac{1}{tp_i}\cdot\ell_i \cdot \eye
		\preceq \frac{\epsilon^2}{c\log(d/\delta)\sum_i\tilde \ell_i}\cdot\frac{\sum_i\tilde\ell_i}{\tilde\ell_i}\ell_i \cdot \eye
		\preceq \frac{\epsilon^2}{c\log(d/\delta)}\eye.
	\end{align*}
	In addition
	\begin{align*}
		\frac{1}{t}\bSig^{-1}\UU^*\BB^\transpose  \BB \UU\bSig^{-1}
		\preceq\frac{1}{t}\bSig^{-1}\UU^*\BB^* \BB \UU\bSig^{-1}
		=\frac{\epsilon^2}{c\log(d/\delta)}\eye.
	\end{align*}
	These two give $\|\XX_j\|\leq\frac{\epsilon^2}{c\log(d/\delta)}$. We then bound the variance of $\YY$:
	\begin{align*}
		\E[\YY\YY^*]&=\E[\YY^*\YY]=t\E[\XX_j\XX_j^*]\\
		&\preceq \frac{1}{t}\sum_{i}p_i\cdot\frac{1}{p_i^2}\bSig^{-1}\UU^* \BB_i^\transpose \BB_i  \UU\bSig^{-2}\UU^*\BB_i^*\bar \BB_i \UU\bSig^{-1}\\
		&\preceq \frac{1}{t}\sum_i \frac{\sum\tilde\ell_i}{\tilde\ell_i}\bSig^{-1}\UU^* \BB_i^* \p{\BB_i \UU\bSig^{-2}\UU^*\BB_i^*} \BB_i \UU\bSig^{-1}\\
		&\preceq \frac{1}{t}\sum_i \frac{\sum\tilde\ell_i}{\tilde\ell_i}\cdot\ell_i\bSig^{-1}\UU^* \BB_i^*  \BB_i \UU\bSig^{-1}\\
		&\preceq \frac{\epsilon^2}{c\log(d/\delta)}\bSig^{-1}\UU^* \BB^*\BB  \UU\bSig^{-1}
		\preceq \frac{\epsilon^2}{c\log(d/\delta)}\eye.
	\end{align*}
	
	By the stable rank matrix Bernstein inequality, we have for large enough $c$:
	\[
		P(\|\YY\|_2>\epsilon)\leq\frac{4\text{tr}(\eye)}{\|\eye\|}\exp\p{-\frac{\epsilon^2/2}{\frac{\epsilon^2}{c\log(d/\delta)}(\|\eye\|+\epsilon/3)}}<\delta,
	\]
	where we use the fact that $tr(\eye)\leq d$ and $\|\eye\|=1$.
	
\end{proof}

\subsection{Proof of \cref{THM:NORMRESULT}}
\begin{proof}
	Let $\BB^*=\UU\bSig\VV^*$ be the SVD of $\BB^*$. We have $\ell_i = \BB_i^*(\UU\bSig^{-2}\UU^*)\BB_i$.
	Let $\YY=\CC^\transpose \CC - \BB^\transpose  \BB $. Then we write
	\[
		\YY=\sum_{j=1}^t \p{\CC_j^\transpose \CC_j -\frac{1}{t}\BB^\transpose  \BB }\triangleq\sum_{j=1}^t \XX_j.
	\]
	Note with probability $p_i$
	\[
		\XX_j = \frac{1}{t}\p{\frac{1}{p_i}\BB_i^\transpose \BB_i -\BB^\transpose  \BB }.
	\]
	Now we bound the variance of $\YY$:
	\begin{align*}
		\E[\YY^*\YY]&=t\E[\XX_j^*\XX_j]\\
		&\preceq\frac{1}{t}\sum_i p_i\frac{1}{p_i^2}\BB_i^* \bar\BB_i\BB_i^\transpose \BB_i\\
		&=\frac{1}{t}\sum_i \frac{\sum_i\tilde\ell_i}{\tilde\ell_i} \|\BB_i\|^2\BB_i^*\BB_i\\
		&\preceq\frac{\epsilon^2}{c\log(d/\delta)}\eye.
	\end{align*}
	By the matrix Chernoff bound \cite{gross2010note}, we have
	\[
		\Pr(\|\YY\|>\epsilon)\leq 2d\exp\p{-\frac{\epsilon^2}{\frac{4\epsilon^2}{c\log(d/\delta)}}}=\cO(\delta).
	\]
	We remark that for our particular task, $\YY\YY^*=\YY^*\YY$. In general this is not true. By applying the non-Hermitian matrix Bernstein inequality in \citet{tropp2015introduction}, one can derive the same result off by a multiplicative constant factor.
\end{proof}

\subsection{Theoretical results on the hybrid randomized-deterministic sampling algorithm}

We first present a useful inequality from \citet{cohen2015optimal} for subspace embeddings in the complex setting.

\begin{lemma}\label{lma:embeddingamm}
	Let $\bSS$ be an $\epsilon$-subspace embedding for $\text{span}(\AA, \BB)$, where $\AA,\BB\in\C^{n\times d}$. Then we have:
	\[
		\|\AA^* \bSS^* \bSS\BB-\AA^* \BB\|\leq \epsilon \|\AA\|\|\BB\|.
	\]	
\end{lemma}

\begin{proof}[Proof of \cref{lma:embeddingamm}]
	W.l.o.g., we assume that $\|\AA\|=\|\BB\|=1$, since we can divide both sides by $\|\AA\|\|\BB\|$. Let $\UU$ be an orthonormal matrix of which the columns form a basis for $\text{span}(\AA,\BB)$. Note since $\|\AA\|=\|\BB\|=1$, for any $\xx,\yy$, we have $\AA\xx=\UU\bs$ and $\BB\yy=\UU\bt$ such that $\|\bs\|\leq \|\xx\|$ and $\|\bt\|\leq \|\yy\|$. Now:
	\begin{align*}
		&\|\AA^*\bSS^*\bSS\BB-\AA^*\BB\|\\
		=&\sup_{\|\xx\|=\|\yy\|=1}\abs{\ip{\bSS\AA\xx}{\bSS\BB\yy}-\ip{\AA\xx}{\BB\yy}}\\
		=&\sup_{\|\bs\|, \|\bt\|\leq 1}\abs{\ip{\bSS\UU\bs}{\bSS\UU\bt}-\ip{\UU\bs}{\UU\bt}}\\
		=&\|\UU^*\bSS^*\bSS\UU-\eye\|
		\leq\epsilon.
	\end{align*}
\end{proof}

We are now ready to prove \cref{THM:HYBRIDUPPERBOUND}.
\begin{proof}[Proof of  \cref{THM:HYBRIDUPPERBOUND}]
	If $\xx\in\text{ker}(\DD_N^{1/2}\AA_N)$, then the statement holds trivially. Assume without loss of generality that  $\xx\not\in\text{ker}(\DD_N^{1/2}\AA_N)$.
	
	We first show that 
	\begin{align} \label{eq:ripsubembedding}
		\|\AA_N^\transpose\DD_N^{1/2}\bSS^\transpose\bSS \DD_N^{1/2} \AA_N-\AA_N^\transpose  \DD_N\AA_N\|_2
		\leq\epsilon \|\AA_N^\transpose  |\DD_N|\AA_N\| .
	\end{align}
	By \cref{lma:embeddingamm}, it suffices to show that $\bSS$ is a subspace embedding for $\text{span}\p{\DD_N^{1/2} \AA_N, (\DD_N^{1/2})^* \AA_N}$. Since $\DD_N^{1/2}\AA_N$ has the relaxed RIP, for $\TT$ being a sampling matrix that randomly samples $t\triangleq O(d^2/\epsilon)$ rows of $\DD_N^{1/2}\AA_N$, we have:
	\[
		\Pr\p{\forall \xx: \|\TT \DD_N^{1/2} \AA_N\xx\|^2=\rho^2(1\pm\epsilon)\|\xx\|^2}\geq1-\frac{1}{n}.
	\]
	Since $\bSS=\sqrt{\frac{n}{t}}\TT$, this leads to 
	\[
		 \|\bSS \DD_N^{1/2} \AA_N\xx\|^2=\rho^2(1\pm\epsilon)\cdot\frac{n}{t}\|\xx\|^2=(1\pm\epsilon)\|\DD_N^{1/2} \AA_N\xx\|^2.
	\]
	The reason for the last step is the following: we randomly partition $\DD_N^{1/2}\AA_N$ into $\frac{n}{t}$ chunks of rows, where each chunk has $t$ rows. Denote the $i^{th}$ chunk as $\DD_{N_i}^{1/2}\AA_{N_i}$ and correspondingly $\AA_{N_i}$. By the relaxed RIP and union bound, we have with probability $1-\frac{1}{t}$ that all $\frac{n}{t}$ chunks have $\|\DD_{N_i}^{1/2}\AA_{N_i}\xx\|^2=(1\pm\epsilon)\rho^2 \|\xx\|^2$. So in total:
	\[
		\|\DD_N^{1/2}\AA_N \xx\|^2=\sum_{i=1}^{n/t} \|\DD_{N_i}^{1/2}\AA_{N_i}\xx\|^2 = (1\pm\epsilon)\frac{n\rho^2}{t} \|\xx\|^2.
	\]
	The same proof holds for showing $\bSS$ is an $\epsilon$-subspace embedding for $\text{span}\p{ (\DD_N^{1/2})^* \AA_N}$.
	
	Let $E=\cup_{i=1}^T E_i$. By \cref{eq:ripsubembedding} and the fact that $c\|\AA_N^\transpose|\DD_N|\AA_N\|\leq \|\sum_i\EE^i\|$:
	\begin{align*}
		&\|\AA_N^\transpose \DD_N^{1/2}\bSS^\transpose\bSS \DD_N^{1/2} \AA_N-\AA_N^\transpose  \DD_N\AA_N\|\\
		=&\|\AA_N^\transpose \DD_N^{1/2}\bSS^\transpose\bSS \DD_N^{1/2} \AA_N-\AA_N^\transpose  \DD_N\AA_N  + \AA_E^\transpose  \DD_E \AA_E- \AA_E^\transpose  \DD_E \AA_E\|\\
		=&\|\sum_{i=1}^T \EE^i+\AA_N^\transpose \DD_N^{1/2}\bSS^\transpose\bSS \DD_N^{1/2} \AA_N-\AA^\transpose \DD\AA \|_2\\
		\leq& \epsilon\|\AA_N^\transpose|\DD_N|\AA_N\|
		\leq  \frac{\epsilon}{c} \|\sum_{i}\EE^i\|\\
		\leq & \frac{\epsilon}{c-1}\p{\|\sum_i\EE^i\|-\|\AA_N^\transpose|\DD_N|\AA_N\|}\\
		\leq  & \frac{\epsilon}{c-1}\p{\|\sum_i\EE^i\|-\|\AA_N^\transpose\DD_N\AA_N\|}\\
		\leq & \frac{\epsilon}{c-1}\|\AA^\transpose \DD\AA \|.
	\end{align*}
\end{proof}

\subsection{Fast Computation of Leverage Scores}\label{SEC:APPENDIX_FASTLS}
Despite the nice properties of leverage scores, they are data-dependent features and quite expensive to compute. In this section, we show how one can efficiently approximate all the leverage scores simultaneously.
\begin{theorem}
	\label{lemma:ls_appr}
	Let $\BB\in\C^{n\times d}$ and let $\bSS\in\C^{s\times n}$ be an $\epsilon$-subspace embedding of $\text{span}(\BB)$. Let $\bSS\BB=\QQ\RR^{-1}$ be a $QR$-factorization of $\bSS\BB$, where $\QQ\in\C^{s\times d}$ has orthonormal columns and $\RR^{-1}\in\C^{d\times d}$. Let $\GG\in\R^{d\times \log n}$ be a random Gaussian matrix. We define the $i^{th}$ approximate leverage score to be:
	$
	\tilde\ell_i = \|\ee_i^\transpose\BB\RR\GG\|^2
	$
	Then $\tilde\ell_i=(1\pm \epsilon) \ell_i$ for all $i$ with high probability, and all $\tilde\ell_i$ can be calculated simultaneously in $\cO\p{(\nnz(\AA)+d^2)\log n}$ time.
\end{theorem}
\begin{proof}
    Define 
    \[
        \ell_i' = \|\ee_i^\transpose\BB\RR\|^2.
    \]
    \begin{itemize}
        \item We first show that $\ell_i'=O(1\pm\epsilon)\ell_i$ for all $i\in[n]$.
        Let $\BB=\UU\bSig\VV^*$. Since $\BB\RR$ has the same column space as $\BB$, we have $\BB\RR=\UU\TT^{-1}$, for some matrix $\TT$. We have:
        \begin{align*}
            & \|\xx\|=\|\QQ\xx\|=\|\bSS\BB\RR\xx\|=(1\pm\epsilon) \|\BB\RR\xx\|.
        \end{align*}
        Hence
        \[
            \|\TT^{-1}\xx\|=\|\UU\TT^{-1}\xx\|=\|\BB\RR\xx\| = (1\pm O(\epsilon))\|\xx\|.
        \]
        This implies that $\TT^{-1}$ is well-conditioned: all singular values of $\TT^{-1}$ are of order $1 \pm O(\epsilon)$. With this property:
        \begin{align*}
            \ell_i'&=\|\ee_i^\transpose \BB\RR\|^2 = (1\pm O(\epsilon))\|\ee_i^\transpose \BB\RR\TT\|^2 \\
            &=(1\pm O(\epsilon))\|\ee_i^\transpose \UU\|^2 =(1\pm O(\epsilon)) \ell_i.
        \end{align*}
        
        \item   
            The second step is to show that $\tilde\ell_i=(1\pm\epsilon)\ell_i'$. Recall the Johnson-Lindenstrauss lemma: let $\GG$ be as defined above. Then for all vectors $\zz\in\C^d$:
            \[
                \Pr\p{\| \zz^\transpose \GG\|^2=(1\pm\epsilon)\|\zz\|^2}\geq 1-\delta.
            \]
            We remark that the JL lemma holds for complex vectors $\zz$ as in \citet{krahmer2011new}. Now set $\zz=\ee_i^\transpose \BB\RR$:
            \[
                \Pr\p{\| \ee_i^\transpose \BB\RR \GG\|^2=(1\pm\epsilon)\|\ee_i^\transpose \BB\RR\|^2}\geq 1-\delta,
            \]
            and we get the desired result.
        \item 
			The time complexity for such a construction is the same as the construction for real matrices, which takes $\cO((\nnz(\AA)+d^2)\log n)$ time. 
    \end{itemize}
\end{proof}

\subsection{Proof of \cref{THM:UPPERBOUNDCOMPARISON}}
\begin{proof}
	By \cref{lma:embeddingamm}, we have that:
	\[
		\|\AA^\transpose  \TT^\transpose\TT \DD\AA - \AA^\transpose \DD\AA \|\leq \epsilon \|\AA\|\|\DD\AA\|,
	\] and
	\begin{align*}
		\|\AA^\transpose  \DD^{1/2} \bSS^\transpose\bSS \DD^{1/2}\AA-\AA^\transpose \DD\AA \|
		\leq \epsilon \|\DD^{1/2}\AA\|\|(\DD^{1/2})^*\AA|\|= \epsilon \|\DD^{1/2}\AA\|_2^2.
	\end{align*}
	
	Note that 
	\begin{align*}
		 &\|\DD^{1/2}\AA\|_2^2 = \lm (\AA^\transpose  (\DD^{1/2})^* \DD^{1/2}\AA)\\
		 =&\lm(\AA^\transpose  |\DD| \AA)
		 =\|\AA^\transpose  |\DD|\AA\|_2\\
		 \leq &\|\AA\| \||\DD|\AA\|
		 =\|\AA\|\|\DD\AA\|.
	\end{align*}
	So sampling in the latter way is always as good as the former.
	
	Now we give a simple example that the first sampling scheme can give an arbitrarily worse bound. Let $\AA=\begin{bmatrix}
		a_1 & 0\\  0& a_2
	\end{bmatrix}$ and $\DD=\begin{bmatrix}
		d_1 & 0 \\ 0 & d_2
	\end{bmatrix}$, where $1=a_1>a_2$ and $1=|d_1|<|d_2|$.
	
	Hence $\DD\AA=\begin{bmatrix}
		a_1d_1 & 0 \\ 0 & a_2d_2
	\end{bmatrix}$ and $\AA^\transpose  |\DD| \AA =\begin{bmatrix}
		|d_1|a_1^2 & 0\\ 0& |d_2|a_2^2.
	\end{bmatrix}$
	
	By the above calculation, $\|\DD^{1/2}\AA\|_2^2=\max\{1, |d_2|a_2^2\}$, and $\|\AA\|\|\DD\AA\|=\max\{1, |d_2|a_2\}$. Let $a_2 = \Theta(\sqrt{1/|d_2|})$ and making $|d_2|$ arbitrarily large, we then have $\|\AA\|\|\DD\AA\|\gg  \|\DD^{1/2}\AA\|_2^2$.
\end{proof}

\subsection{Fast Local Convergence of NEWTON-CG}
\begin{theorem}[Fast Local Convergence]
	\label{thm:newton_cg_formal}
	Let $\bSS$ be the leverage score sampling matrix as in \cref{THM:LOWNERRESULT} with precision $\epsilon$. Let $ r(\xx) = \lambda \vnorm{\xx}^{2}/2 $ and $\lambda \geq 4 \vnorm{\AA}^{2} h$ where  $h$ is the Lipschitz continuity constant of the derivative, i.e., $ |f^{''}_{i}(t)| \leq h $ for some $ h < \infty $. Then for sub-sampled Newton-CG with initial point satisfying $ \vnorm{\xxo - \xxs} \leq \mu/(4 L) $, step-size $\alphak = 1 ,$ and the approximate Hessian $\HH = \AA^\transpose \DD^{1/2}\bSS^\transpose\bSS\DD^{1/2}\AA + \lambda \eye$ , we have the following error recursion $
	\left\|\xxk-\xxs\right\| \leq  C_{q} \cdot\left\|\xxk-\xxs\right\|^{2}+C_l \cdot\left\|\xxk-\xxs\right\|$, where $ \xxs $ is the optimal solution, $C_{q}=\frac{2 L}{(1-O(\epsilon)) \mu}$, $C_{l}=\frac{3 \epsilon}{1-O(\epsilon)} \sqrt{\kappa}$, $L$ is the Lipschitz continuity constant of the Hessian,  $\mu=\lambda_{\min }\left(\nabla^{2} F\left(\xxs\right)\right)>0, \quad \nu=\lambda_{\max }\left(\nabla^{2} F\left(\xxs\right)\right)<\infty$, and $\kappa=\nu/\mu$ is the condition number.
\end{theorem}

\begin{proof}
	Let $\BB=\DD^{1/2}\AA$, $\bSS$ be the sketching matrix, and $\CC=\bSS\DD^{1/2}\AA$. By \cref{THM:LOWNERRESULT}, we have:
	\begin{align}\label{eq:regularizedlowner}
		-\epsilon \AA^\transpose|\DD|\AA\preceq \AA^\transpose\DD^{1/2}\bSS^\transpose\bSS\DD^{1/2}\AA - \AA^\transpose\DD\AA\preceq \epsilon \AA^\transpose|\DD|\AA.
	\end{align}
	Rewrite 
	\begin{align*}
		\AA^\transpose|\DD|\AA &= \sum_{i=1}^n |\DD_{i,i}|\AA^\transpose_i \AA =\sum_{i=1}^n \DD_{i,i}\AA^\transpose_i \AA - 2\sum_{i:\DD_{i,i}<0} \DD_{i,i}\AA^\transpose_i \AA\preceq\AA^\transpose\DD\AA + \QQ,
	\end{align*}
	where $\QQ\triangleq \lambda \eye$ as defined in the theorem. The above inequality then holds by the definition of $\lambda$. Therefore, by \cref{eq:regularizedlowner} we have
	\[
			-\epsilon (\AA^\transpose\DD\AA+\QQ)\preceq (\AA^\transpose\DD^{1/2}\bSS^\transpose\bSS\DD^{1/2}\AA+\QQ) - (\AA^\transpose\DD\AA+\QQ)\preceq \epsilon( \AA^\transpose\DD\AA+\QQ).
	\]
	This form satisfies the fast convergence condition in \cite[Lemma 7]{xu2016sub}.  Applying their lemma leads to our conclusion.
\end{proof}
\section{Sketching for Optimization--More Details and Experiments}
\label{SEC:APPENDIX:EXPERIMENTS}
\paragraph{More Background on Some Optimization Methods}
\begin{itemize}[label = \textbf{--}]
\item 
		\textbf{Convex Optimization: Sub-sampled Newton-CG.}
In strongly convex settings where $ \nabla^{2} F(\xx) \succeq \mu \eye$ for some $ \mu > 0 $, the Hessian matrix is positive definite, and the $ k\th $ iteration of the sub-sampled Newton-CG method is often written as $\xxkk = \xxk + \alphak \ppk$, where $ \ppk $ is an approximate solution to the linear system $ \HHk \pp = - \nabla F(\xxk) $, 
obtained using the conjugate gradient (CG) algorithm \cite{saad2003iterative}, 
and $  0 < \alpha_k \leq 1 $ is an appropriate step-size, 
which satisfies the Armijo-type line search \cite{nocedal2006numerical} condition stating that $F(\xxkk) \leq F(\xxk) + \rho \alpha_k \dotprod{\ppk, \bggk}$, where $ 0< \rho < 1 $ is a given line-search parameter (see \cref{alg:newton_cg} in \cref{SEC:APPENDIX:ALG}). 
\item	
	\textbf{Non-convex Optimization: Sub-sampled Newton-MR.}
In non-convex settings, the Hessian matrix could be indefinite and possibly rank-deficient. In light of this, in the $ k\th $ iteration, Newton-MR \cite{roosta2018newton} with an approximate Hessian involves iterations of the form $ \xxkk = \xxk + \alphak \ppk $ where $\ppk \approx -\HHdk \nabla F(\xxk)$ is  obtained by a variety of least-squares iterative solvers such as MINRES-QLP \cite{choi2011minres}, and $ 0 < \alpha_k \leq 1 $ is such that $\vnorm{\bggkk}^{2} \leq \vnorm{\bggk}^{2} + 2 \rho \alpha_k \dotprod{\ppk, \HHk \bggk}$ (see \cref{alg:newton_mr} in \cref{SEC:APPENDIX:ALG}). 
It has been shown that Newton-MR achieves fast local and global convergence rates when applied to a class of non-convex problems known as invex \cite{roosta2018newton}, whose stationary points are global minima. 
From \citet[Corollary 1]{liu2019stability} with $ \epsilon $ small enough in \cref{eq:H}, \cref{alg:newton_mr} converges to an $\epsilon_{g}$-approximate first-order stationary point $ \vnorm{\nabla F(\xxk)} \leq \epsilon_{g} $ in at most $ k \in \mathcal{O}\left(\log \left( 1/ \epsilon_{g} \right) \right) $ iterations. 
Every iteration of MINRES-QLP requires one Hessian-vector product, which using the full Hessian, amounts to a complexity of $ \mathcal{O}\left( \nnz(\AA) \right)$. In the worst case, MINRES-QLP requires $ \mathcal{O}(d) $ iterations to obtain a solution. Putting this all together, the overall running time of Newton-MR with exact Hessian to achieve an $\epsilon_{g}$-approximate first-order stationary point is $ k \in \mathcal{O}\left(\nnz(\AA) d \log \left( 1/ \epsilon_{g} \right) \right) $. However, with the complex leverage score sampling of \cref{alg:lssampling} (cf.\ \cref{lemma:ls_appr}), the running time then becomes $ k \in \mathcal{O}\left( \left( \nnz(\AA) \log n + d^{3} \right) \log \left( 1/ \epsilon_{g} \right) \right) $. 

\item 
	\textbf{Non-convex Optimization: Sub-sampled Trust Region.}
As a more versatile alternative to line-search, trust-region \cite{sorensen1982newton,conn2000trust} is an elegant globalization strategy that has attracted much attention. Recently,  \citet{xuNonconvexTheoretical2017} theoretically studied the variants of trust-region in which the Hessian is approximated as in \cref{eq:H}. The crux of each iteration of the resulting algorithm is the (approximate) solution to a constrained quadratic sub-problem of the form $\min_{\|\pp\|\le \Delta_{k}} m_k(\pp) \triangleq \dotprod{\nabla F(\xxk), \pp} + \frac{1}{2}\dotprod{\pp, \HHk \pp}$, for which a variety of methods exists, e.g, CG-Steihaug \cite{steihaug1983conjugate,toint1981towards}, and the generalized Lanczos based methods \cite{gould1999solving, lenders2016trlib} (see \cref{alg:tr} in \cref{SEC:APPENDIX:ALG}). Suppose for $i\in[n]$, $\vnorm{\aa_{i}}^{2} \sup_{\xx \in \real^{d}} |f_{i}^{\prime \prime}(\xx)| \le K_{i}$ and define $ K_{\max} \triangleq \max_{i=1,\ldots,n} K_{i}$, $\widehat K \triangleq \sum_{i=1}^n K_i/n$. 
By considering uniform and row-norm sampling of  $ \DD^{1/2} \AA $ with respective sampling complexities of $ |\mathcal S| \in \mathcal{O}(K_{\max}^{2} \epsilon^{-2} \log d ) $ and $ |\mathcal S| \in \mathcal{O}( \widehat{K}^2 \epsilon^{-2} \log d )$, \citet{xuNonconvexTheoretical2017} showed that one can guarantee \cref{eq:H} with high-probability, and as a result \cref{alg:tr} achieves an optimal iteration complexity, i.e., it converges to an $(\epsilon_{g},\epsilon_{\HH})$-approximate second-order stationary point $ \vnorm{\nabla F(\xxk)} \leq \epsilon_{g} $ and $ \lambda_{\text{min}}(\nabla^{2} F(\xxk)) \geq -\epsilon_{\HH} $ in at most $ k \in \mathcal{O}(\max\{\epsilon_{g}^{-2} \epsilon_{\HH}^{-1},\epsilon_{\HH}^{-3}\}) $ iterations. 
\end{itemize}

\paragraph{Sub-sampling Schemes.}
Recall the following terms:
\begin{itemize}[label = \textbf{--}]
	\item \textbf{Uniform}: For this sampling, we have $ p_{i} = 1/n, \; i = 1,\ldots,n $.
	\item \textbf{Leverage Score (LS)}: Complex leverage score sampling by considering the leverage scores of $ \DD^{1/2}\AA $ as in \cref{alg:lssampling}.
	\item \textbf{Row Norm (RN)}: Row-norm sampling of $ \DD^{1/2} \AA $ using \cref{eq:togethersampling} where $s((\DD^{1/2} \AA)_i)=|f_{i}{''}(\dotprod{\aa_{i}, \xx})|\vnorm{\aa_i}_{2}^{2}$
	\item \textbf{Mixed Leverage Score (LS-MX)}: A mixed leverage score sampling strategy arising from a non-symmetric viewpoint of the product $ \AA^{\transpose} \left( \DD \AA\right) $ 
    using \cref{eq:separatesampling} with $s(\AA_i)=\ell_i(\AA)$ and $S((\DD\AA)_i)=\ell_i(\DD\AA)$.
	\item \textbf{Mixed Norm Mixture (RN-MX)}: A mixed row-norm sampling strategy with the same non-symmetric viewpoint as in \cref{eq:separatesampling} with $s(\AA_i)=\vnorm{(\AA)_i}$ and $S((\DD\AA)_i)=\vnorm{(\DD\AA)_i}$.
	\item \textbf{Hybrid Randomized-Deterministic (LS-Det)}: Hybrid deterministic-leverage score sampling of \cref{alg:lsdetsampling}. 
	\item \textbf{Full}: In this case, the exact Hessian is used.
\end{itemize}

\paragraph{Datasets.} The datasets used in our experiments for this section are listed in \cref{table:data}. All datasets are publicly available from the UC Irvine Machine Learning Repository \cite{Dua:2019}. 
\begin{table}[!htbp]
	\centering
	\begin{tabular}{|c|c|c|} 
		\hline
		Name & $ n $ & $ d $\\ [0.5ex] 
		\hline \hline && \\ [-2ex]
		\texttt{Drive Diagnostics} & 50,000 & 48\\ [1ex] 
		\hline && \\ [-2ex]
		\texttt{covertype}, & 581,012 & 54\\ [1ex] 
		\hline && \\ [-2ex]
		\texttt{UJIIndoorLoc} & 19,937 & 520\\ [1ex] 
		\hline
	\end{tabular}
	\caption{Data sets used for our experiments.\label{table:data}}
\end{table}

\paragraph{Hyper-parameters.} \cref{alg:newton_cg,alg:newton_mr,alg:tr} are always initialized at $ \xxo = \bm{0} $. In all of our experiments, we run each method until either a maximum number of iterations or a maximum number of function evaluations is reached. The maximum number of CG iterations within Newton-CG, MINRES-QLP iterations within Newton-MR and CG-Steihaug within trust-region methods are all set to $100$. The parameter of line-search $ \rho $ in Newton-MR is set to $ 10^{-4} $. For trust-region, we set $ \Delta_0 = 1 $, $ \eta = 0.8 $ and $ \gamma = 1.2 $.

\paragraph{Performance Evaluation.}
In all of our experiments, we plot the objective value or the gradient norm vs.\ the total
number of oracle calls of function, gradient,  and Hessian-vector products. This is because comparing algorithms
in terms of ``wall-clock'' time can be highly affected by their particular implementation
details as well as system specifications. In contrast, counting the number of oracle calls,
as an implementation and system independent unit of complexity, is most appropriate and fair.
More specifically, after computing each function value, computing the corresponding gradient is equivalent to one additional function evaluation. Our implementations are Hessian-free, i.e., we merely require Hessian-vector products instead of using the explicit Hessian. For this, each Hessian-vector product involving $ \AA \DD \AA $ amounts to two additional function evaluations, as compared with gradient evaluation. In this light, each matrix-vector product involving $ \DD^{1/2} \AA $ for approximating the underlying complex leverage scores is equivalent to one gradient evaluation.

Following the theory of Newton-MR, whose convergence is measured by the norm of the gradient, we evaluate \cref{alg:newton_mr} with various sampling schemes by plotting $ \vnorm{\nabla F(\xxk)} $ vs.\ the total number of oracle calls, whereas for \cref{alg:newton_cg,alg:tr}, which guarantees descent in objective function, we plot $ F(\xxk) $ vs.\ the total number of oracle calls.

\subsection{Comparison Among Various Sketching Techniques}
To verify the result of \cref{THM:UPPERBOUNDCOMPARISON}, in this section we present empirical evaluations of Uniform, LS, RN, LS-MX, RN-MX and Full in the context of \cref{alg:newton_cg,alg:newton_mr,alg:tr}. The results are depicted in \cref{fig:newton_cg,fig:newton_mr,fig:trust_region}. It can be clearly seen that for both algorithms, LS and LS-MX sampling amounts to a more efficient algorithm than that with RN and RN-MX variants, and at times this difference is more pronounced than other times.

\begin{figure}[H]
	\centering
	\subfigure[$ F(x_{k}) $ vs.\ Oracle calls (\texttt{Drive Diagnostics})]
	{\includegraphics[scale=0.4]{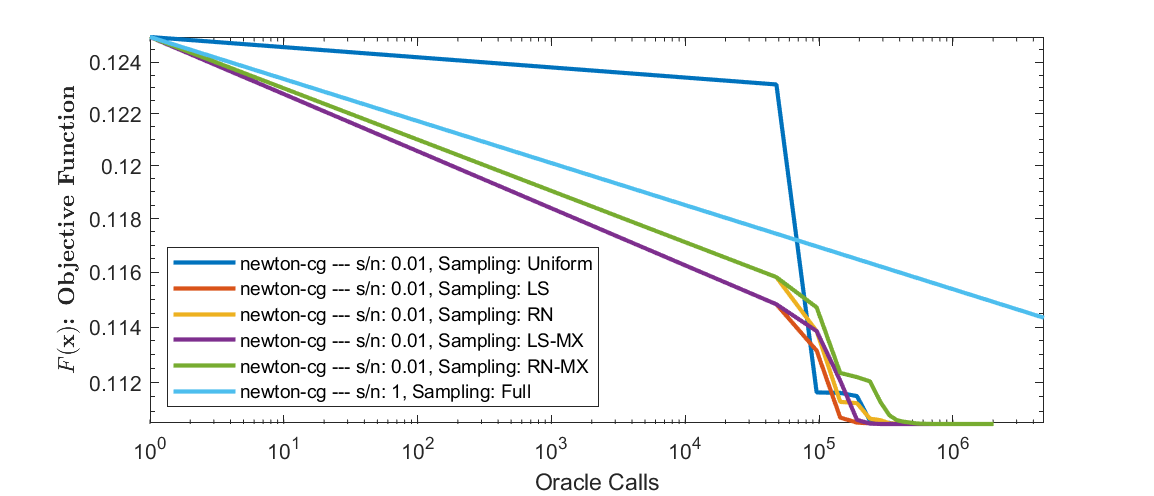}}
	\subfigure[$ F(x_{k}) $ vs.\ Oracle calls (\texttt{Cover Type})]
	{\includegraphics[scale=0.4]{./figs/nlls_convex/covetype/newton-cg/Obj_Props.png}}
	\subfigure[$ F(x_{k}) $ vs.\ Oracle calls (\texttt{UJIIndoorLoc})]
	{\includegraphics[scale=0.4]{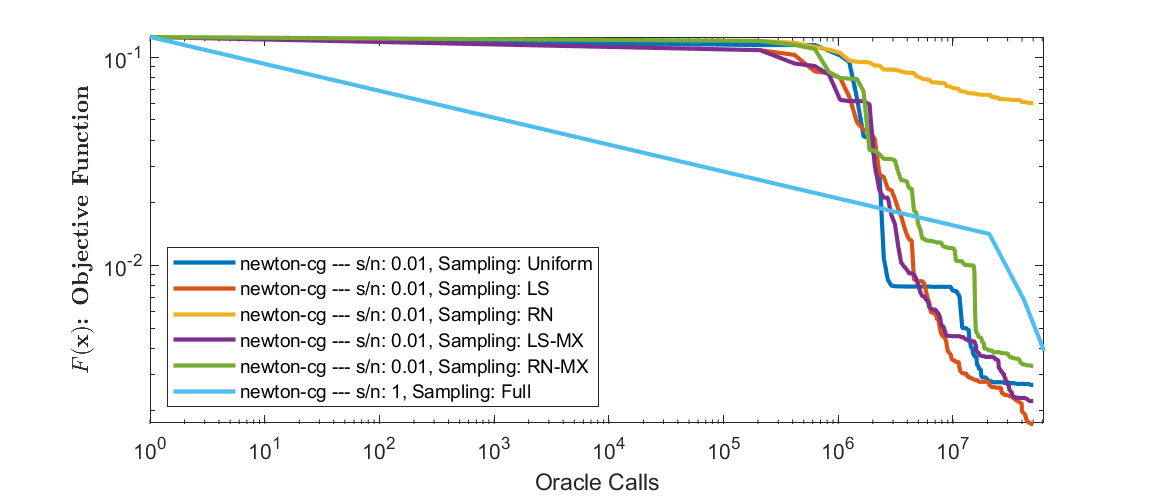}}
	\caption{Comparison of Newton-CG (\cref{alg:newton_cg}) using various sampling schemes. \label{fig:newton_cg}}
\end{figure}


\begin{figure}[H]
	\centering
	\subfigure[$ F(x_{k}) $ vs.\ Oracle calls (\texttt{Drive Diagnostics})]
	{\includegraphics[scale=0.4]{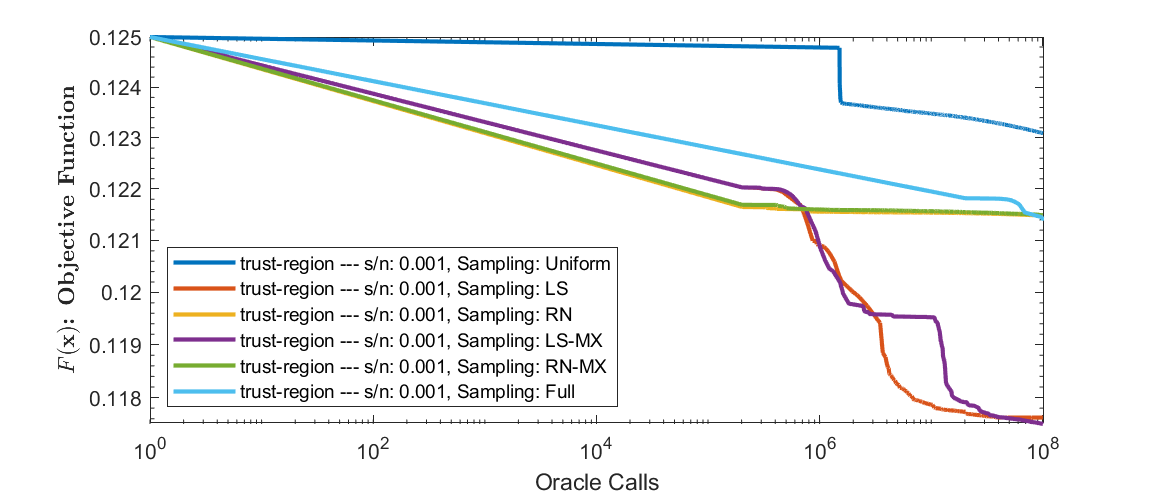}}
	\subfigure[$ F(x_{k}) $ vs.\ Oracle calls (\texttt{Cover Type})]
	{\includegraphics[scale=0.4]{./figs/nlls_none/covetype/trust-region/Obj_Props.png}}
	\subfigure[$ F(x_{k}) $ vs.\ Oracle calls (\texttt{UJIIndoorLoc})]
	{\includegraphics[scale=0.4]{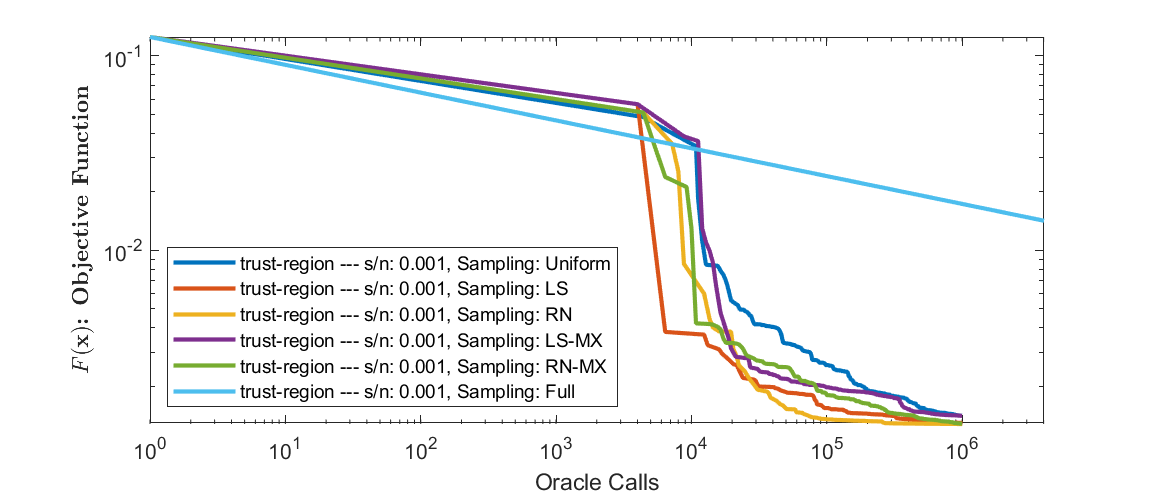}}
	\caption{Comparison of Trust-region (\cref{alg:tr}) using various sampling schemes. \label{fig:trust_region}}
\end{figure}

\subsection{Evaluation of Hybrid Sketching Techniques}
Here, to verify the result of \cref{THM:HYBRIDUPPERBOUND}, we evaluate the performance of \cref{alg:tr} by varying the terms involved in $ \EE $. We do this for a simple splitting of $ \HH = \EE + \NN $, i.e., $ T=1 $ in \cref{THM:HYBRIDUPPERBOUND}. We fix the overall sample size and change the fraction of samples that are deterministically picked in $ \EE $. The results are depicted in \cref{fig:hybrid}. The value in brackets in front of LS-Det is the fraction of samples that are included in $ \EE $, i.e., deterministic samples. ``LS-Det (0)'' and ``LS-Det (1)'' correspond to $ \EE = \bm{0}$ and $ \NN = \bm{0} $, respectively. The latter strategy has been used in low rank matrix approximations \cite{mccurdy2018ridge}. As it can be seen, the hybrid sampling approach is always competitive with, and at times significantly better than, LS-Det (0). As expected, LS-Det (1), which amounts to entirely deterministic samples, consistently performs worse. This can be easily attributed to the high bias of such a deterministic estimator.

\begin{figure}[H]
	\centering
	\subfigure[$ F(x_{k}) $ vs.\ Oracle calls (\texttt{Drive Diagnostics})]
	{\includegraphics[scale=0.4]{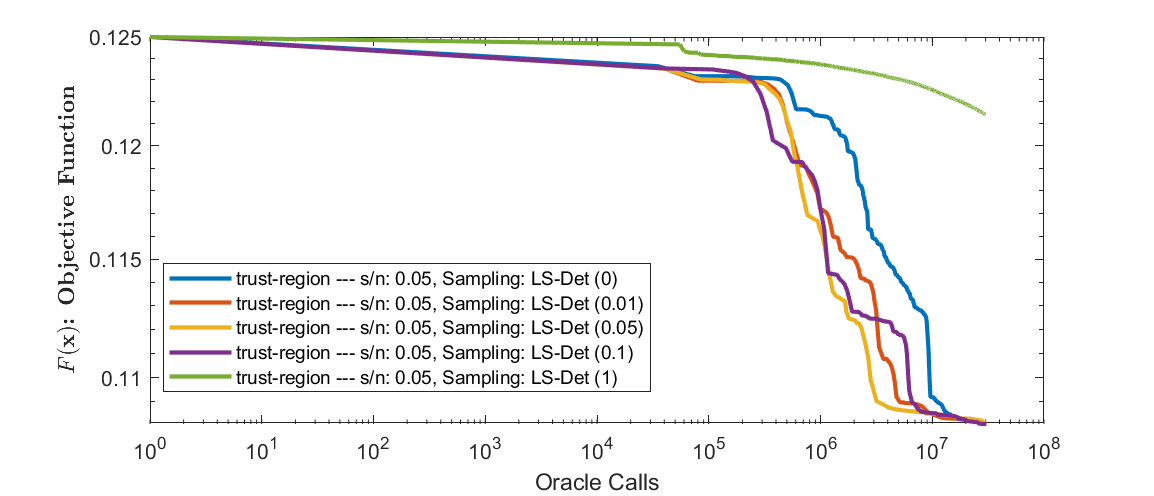}}
	\subfigure[$ F(x_{k}) $ vs.\ Oracle calls (\texttt{Cover Type})]
	{\includegraphics[scale=0.4]{./figs/nlls_det/covetype/trust-region/Obj_Props.png}}
	\subfigure[$ F(x_{k}) $ vs.\ Oracle calls (\texttt{UJIIndoorLoc})]
	{\includegraphics[scale=0.4]{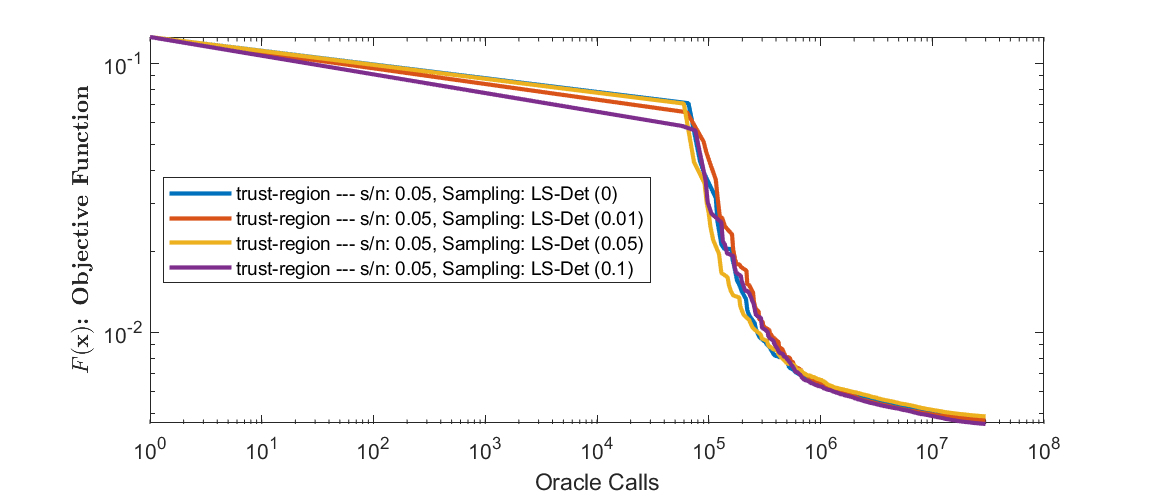}}
	\caption{Comparison of \cref{alg:tr} using hybrid randomized-deterministic sampling schemes. For all runs, the overall sample/mini-batch size for estimating the Hessian matrix is $ s = 0.05n$. The values in parentheses in front of LS-Det is the fraction of samples that are taken deterministically and included in $ \EE $. \label{fig:hybrid}}
\end{figure}

\end{appendices}
\end{document}